\documentclass[letterpaper, journal]{IEEEtran}  

\usepackage{graphicx} 
\usepackage{float}    
\usepackage{verbatim} 
\usepackage{amsmath}  
\usepackage{amssymb}  
\usepackage{amsthm}
\usepackage[font=footnotesize]{caption}
\usepackage[ruled,vlined,linesnumbered]{algorithm2e}
\usepackage{multirow}
\usepackage{array}
\usepackage{booktabs}
\usepackage{hyperref}
\usepackage{gensymb}
\usepackage[dvipsnames]{xcolor}
\usepackage{subcaption}

\newtheorem{theorem}{Theorem}

\theoremstyle{definition}
\newtheorem{assumption}{Assumption}

\theoremstyle{definition}
\newtheorem{definition}{Definition}

\theoremstyle{definition}

\title{A Sequential Quadratic Programming Approach to the Solution of Open-Loop Generalized Nash Equilibria for Autonomous Racing}

\author{Edward L. Zhu and Francesco Borrelli
\thanks{Edward L. Zhu, and Francesco Borrelli {\{\tt\small edward.zhu, fborrelli\}@berkeley.edu} are with the Department of Mechanical Engineering at the University of California, Berkeley CA, USA}
}

\begin{document}

\maketitle
\thispagestyle{empty}
\pagestyle{empty}

\begin{abstract}

Dynamic games can be an effective approach for modeling interactive behavior between multiple competitive agents in autonomous racing and they provide a theoretical framework for simultaneous prediction and control in such scenarios. In this work, we propose DG-SQP, a numerical method for the solution of local generalized Nash equilibria (GNE) for open-loop general-sum dynamic games for agents with nonlinear dynamics and constraints. In particular, we formulate a sequential quadratic programming (SQP) approach which requires only the solution of a single convex quadratic program at each iteration. The three key elements of the method are a non-monotonic line search for solving the associated KKT equations, a merit function to handle zero sum costs, and a decaying regularization scheme for SQP step selection. We show that our method achieves linear convergence in the neighborhood of local GNE and demonstrate the effectiveness of the approach in the context of head-to-head car racing, where we show significant improvement in solver success rate when comparing against the state-of-the-art PATH solver for dynamic games. An implementation of our solver can be found at \url{https://github.com/zhu-edward/DGSQP}.

\end{abstract}


\section{Introduction}

Autonomous racing has steadily gained attention as an area of research due to the unique combination of challenges that it presents \cite{betz2022autonomous}. Like many widely studied autonomous navigation problems, racing scenarios typically involve multiple agents with no information sharing who are subject to coupling through safety related constraints. However, what sets racing apart is the aggressive and direct competition which necessitates performance at the limit (of speed, tire friction, thrust, etc.) in order to achieve victory against one's opponents. 
As such, we are faced with the challenge in prediction and planning for autonomous racing where, we need to solve for interactive multi-agent behavior arising from direct competition under shared constraints, while subject to performance limiting factors such as friction limits from non-linear tire models. 

Approaches to tackling these challenges fall broadly into data-driven and model-based approaches, which have both seen impressive successes. In \cite{wurman2022outracing}, autonomous agents were trained using deep reinforcement learning (DRL) in a video game environment and achieved superior performance in head-to-head races against human players. While this approach was certainly able to learn complex vehicle dynamics and highly expressive multi-agent interaction models, it required days of training and, as with most RL-based methods, is dependent on large quantities of data to reach super-human performance in a simulation environment. In \cite{kaufmann2023champion}, DRL was used to achieve super-human performance in real-world drone racing. However, due to the sparsity of close-proximity interactions in drone races, no interaction models were used in the solution. On the other hand, \cite{betz2023tum,jung2023autonomous} used model-based approaches in a predict-then-plan architecture to achieve real-world racing of full scale race cars at speeds of up to 270 km/h. In these works, an upstream prediction module provides behavior forecasts of the opponents, which are then treated as constant during the downstream planning phase when an action plan is formulated for the ego vehicle (EV) via optimal control techniques which leverage explicit nonlinear vehicle models. These approaches have the benefit of being highly transparent, which is crucial for deployment on expensive and safety critical hardware, but can be limited in expressiveness in terms of interactivity between the opponent predictions and ego plan. Our work seeks to address this limitation for model-based approaches by producing highly interactive solutions to the prediction and planning problem for autonomous racing while maintaining transparency through the use of explicit vehicle models. 

Recent work has posed the prediction and planning problem as an equilibrium finding problem for a dynamic game \cite{zhu2023sequential, spica2020real, schwarting2021stochastic}. These approaches search for equilibria over the joint trajectory space of the EV and its opponents, which allows for simultaneous prediction and planning, as they are not conditioned on a static prior EV plan and can in fact solve over the space of multi-agent interactions. 
Such methods are based on the theory of non-cooperative dynamic games \cite{bacsar1998dynamic} and have been applied to a wide variety of trajectory optimization tasks for multi-agent robotic systems \cite{laine2021multi, lecleach2020algames, fridovich2020efficient, liu2023learning, peters2023contingency, kavuncu2021potential}. Two solution or equilibrium concepts are common for dynamic games, namely the Stackelberg and Nash equilibria, which make different assumptions about the information structure of the game. A Stackelberg equilibrium can be found for a game with an explicit leader-follower hierarchy \cite{fisac2019hierarchical}, whereas a Nash equilibrium models the case when agents make their decisions simultaneously. Our work focuses on the selection of Nash equilibria, which we believe to be a good fit for modeling the behavior of ego-centric agents in racing scenarios where no \emph{a priori} structure is imposed on the order of the interactions and where agents can have zero-sum terms in their objective functions, which describe direct competition with their opponents. 

Motivated by the considerations above, this paper focuses on   finding generalized Nash equilibria (GNE) \cite{facchinei2010generalized} for open-loop dynamic games with state and input constraints with a numerically robust solver. 
We build on our previous conference paper \cite{zhu2023sequential} and present three main contributions. The first is the DG-SQP solver, which is a novel iterative approach for finding GNE of a discrete-time open-loop dynamic game based on sequential quadratic programming (SQP). In particular, the method is able to account for nonlinear game dynamics and constraints on both the game state and agent actions. The three key elements of the method are a non-monotonic line search for solving the associated KKT equations, a merit function to handle zero sum costs, and a decaying regularization scheme for the SQP step selection. The second contribution is a novel application of model predictive contouring control in the context of dynamic games which is used to approximate Frenet-frame kinematics for improved numerical robustness. Finally, we perform an extensive simulation study comparing the performance of our solver to the to the state-of-the-art PATH solver based on mixed complementary problems \cite{dirkse1995path}. We show comparable performance to the PATH solver, in terms of success rate, on dynamic games formulated using exact Frenet-frame kinematics and significant improvements when the approximate formulation is used. 

We remark that the broader objective of our research extends to leveraging these open-loop dynamic game solutions within feedback-based systems. While the focus of this paper is the robustness of the proposed open-loop GNE solver, solver robustness in itself is not the ultimate goal. Rather, it is critical property which allows for effective integration of the solver and/or its solutions into closed-loop systems. This transition from open to closed-loop application of GNE solutions from our proposed solver has previously appeared in \cite{zhu2023sequential}.

\section{Related Work}

\subsection{Numerical Solvers for Nash Equilibria of Dynamic Games}

For the solution of GNE for dynamic games, many numerical approaches have been proposed. \cite{fridovich2020efficient} takes a differential dynamic programming approach to obtain a linear quadratic approximation of the dynamic game and its associated feedback Nash equilibrium. This is built upon in \cite{schwarting2021stochastic} to address stochastic games over belief spaces, which takes into account the effect of noisy measurements on state estimation in the game dynamics. However, these approaches are unable to explicitly account for inequality constraints and instead must include them in the cost function via barrier functions. This can obfuscate the meaning of the game as the cost measures both performance and constraint violation. \cite{kavuncu2021potential,bhatt2023efficient} propose an approach for the subclass of potential dynamic games which was shown to be computationally efficient. However, by requiring that the dynamic game be described by a single potential function over all agents, we note that direct competition, in the form of zero-sum terms in the agents' objective functions, cannot be captured in this approach. \cite{spica2020real,wang2019gamea,wang2019gameb,wang2021game} formulate a method akin to block coordinate descent called Sensitivity Enhanced Iterative Best Response (SE-IBR) where the optimal control problems for each agent, which are coupled in the dynamic game setting, are decoupled but augmented with a term in the objective function which captures the sensitivity of the TV's solution w.r.t. the EV's solution. Agents then improve their solutions in a sequential manner while holding the behavior of all other agents fixed. It was shown that fixed points of this algorithm correspond to GNE. However, each iteration of the method requires the solution of the same number of optimization problems as there are agents and can be slow to converge in practice. In contrast, our proposed solver only requires the solution of a single convex optimization problem at each iteration. \cite{lecleach2020algames} proposes a solver based on the augmented Lagrangian method. The solver, called ALGAMES, shows good performance when compared to \cite{fridovich2020efficient}. However, we show in prior work that ALGAMES appears to struggle with convergence in the context of car racing where more complex dynamics and environments are introduced \cite{zhu2023sequential}. Our proposed solver is perhaps most similar to \cite{laine2023computation} which also leverages SQP for the computation of \emph{feedback} Nash equilibria, but does not investigate its local behavior. Compared to \cite{laine2023computation}, we also introduce an approximation scheme which improves solver convergence in racing scenarios. Finally, we note that the problem of finding GNE for open-loop dynamic games can be formulated as a mixed complementarity problem for which the PATH solver is the state-of-the-art \cite{dirkse1995path}. This solver was used to great effect in multiple recent works to compute GNE for multi-agent navigation tasks \cite{liu2023learning,peters2023contingency}. As such, the PATH solver will be the main target of our comparisons, where we show, through a numerical study, that our proposed solver achieves greater success rate in the context of head-to-head racing scenarios.

\subsection{Game-Theoretic Methods in Autonomous Racing}

Due to its competitive nature, autonomous racing has been a popular test bed for game-theoretic methods of prediction and planning and many of the previously mentioned methods have been applied in this context. \cite{schwarting2021stochastic} formulates the racing problem as a stochastic dynamic game which can take into account state and measurement uncertainty when solving for a Nash equilibrium. However, this approach is unable to explicitly incorporate safety-critical track boundary and collision avoidance constraints, which are crucial in racing. \cite{jia2023rapid} poses the problem as a constrained potential game. This results in good computational efficiency for game-theoretic prediction and planning in a real-time, model predictive game play (MPGP) manner, but importantly precludes zero-sum components in the agents' cost functions by definition of the potential game. This makes it difficult to capture direct competition between the agents and instead the approach relies on heuristics and mode switching in order to induce competitive behavior. \cite{spica2020real,wang2021game,wang2019gamea,wang2019gameb} leverage the aforementioned SE-IBR algorithm to approximately solve a dynamic game in a MPGP manner and use the approximate GNEs as high-level plans for a tracking policy in hardware races. However, this work, like all of the ones discussed so far in this section, have only shown GNE solution results for linear models such as the integrator model or simple nonlinear models such as the kinematic bicycle model. The only exception to this is in \cite{liniger2019noncooperative} where a bimatrix game was defined over the incurred costs of sampled trajectory rollouts for two vehicles each using the dynamic bicycle model. The trajectory pair corresponding to the NE of this bimatrix game was then chosen as the opponent prediction and EV plan. This approach can be considered as a zeroth order method for the solution of GNEs as it only requires evaluations of the components of the dynamic game. (This is in contrast to \cite{lecleach2020algames,dirkse1995path} and our approach, which leverage gradients of the game components). As such, it provides a simple framework for leveraging high fidelity vehicle models. However, it is straightforward to see that the effectiveness of this approach largely depends on the number of trajectory samples taken and it is unclear if the trajectory pair corresponding to the NE of the bimatrix game is a GNE of the dynamic game.

\section{Problem Formulation}

Consider an $M$-agent, finite-horizon, discrete-time, general-sum, open-loop, dynamic game whose state is characterized by the joint dynamical system:
\begin{align} \label{eq:joint_dynamics}
    x_{k+1} = f(x_k, u_k),
\end{align}
where $x_k^i \in \mathcal{X}^i$ and $u_k^i \in \mathcal{U}^i$ are the state and input of agent $i$ at time step $k$  and
\begin{align*}
    x_k &= \begin{bmatrix} {x_k^1}^\top, \dots, {x_k^M}^\top \end{bmatrix}^\top \in \mathcal{X}^1 \times \dots \times \mathcal{X}^M = \mathcal{X} \subseteq \mathbb{R}^n \\
    u_k &= \begin{bmatrix} {u_k^1}^\top, \dots, {u_k^M}^\top \end{bmatrix}^\top \in \mathcal{U}^1 \times \dots \times \mathcal{U}^M = \mathcal{U} \subseteq  \mathbb{R}^m,
\end{align*}
are the concatenated states and inputs of all agents. In this work, we will use the notation $x_k^{\neg i}$ and $u_k^{\neg i}$ to denote the collection of states and inputs for all but the $i$-th agent.

Each agent $i$ attempts to minimize its own cost function, which is comprised of stage costs $l_k^i$ and terminal cost $l_N^i$, over a horizon of length $N$:
\begin{subequations} \label{eq:agent_cost}
\begin{align} 
    \bar{J}^i(\mathbf{x}, \mathbf{u}^i) &= \sum_{k = 0}^{N-1} l_k^i(x_k, u_k^i) + l_N^i(x_N) \label{eq:agent_cost_xu} \\
    &= J^i(\mathbf{u}^1, \dots, \mathbf{u}^M, x_0), \label{eq:agent_cost_u}
\end{align}
\end{subequations}
where $\mathbf{x} = \{x_0, \dots, x_N\}$ and $\mathbf{u}^i = \{u_0^i, \dots, u_{N-1}^i\}$ denote state and input sequences over the horizon. Note that the cost in \eqref{eq:agent_cost_xu} for agent $i$ depends on its \emph{own} inputs and the \emph{joint} state, which can capture dependence on the behavior of the other agents. We arrive at \eqref{eq:agent_cost_u} by recursively substituting in the dynamics \eqref{eq:joint_dynamics} to the cost function, which are naturally a function of the open-loop input sequences for all agents. The agents are additionally subject to $n_c$ constraints
\begin{align} \label{eq:joint_constraints}
    C(\mathbf{u}^1, \dots, \mathbf{u}^M, x_0) \leq 0,
\end{align}
which can be used to describe individual constraints as well as coupling between agents and where we have once again made the dependence on the joint dynamics implicit. For the sake of brevity, when focusing on agent $i$, we omit the inital state $x_0$ and write the cost and constraint functions as $J^i(\mathbf{u}^i,\mathbf{u}^{\neg i})$ and $C(\mathbf{u}^i,\mathbf{u}^{\neg i})$. Let us now define the conditional constraint set
\begin{align*}
    \mathcal{U}^i(\mathbf{u}^{\neg i}) =& \{ \mathbf{u}^i \ | \ C(\mathbf{u}^i, \mathbf{u}^{\neg i}) \leq 0 \},
\end{align*}
which can be interpreted as a restriction of the joint constraint set for agent $i$ given some $\mathbf{u}^{\neg i}$. We make the following assumption about the functions and constraint sets.
\begin{assumption} \label{asm:compact_set_differentiability}
    The sets $\mathcal{X}^i$ and $\mathcal{U}^i$ are compact and the functions $f$, $J^i$, and $C$ are three times continuously differentiable on $\mathcal{X}$ and $\mathcal{U}$ for all $i \in \{1, \dots, M\}$.
\end{assumption}

\subsection{Generalized Nash Equilibrium}

We define the constrained dynamic game as the tuple:
\begin{align} \label{eq:dynamic_game}
    \Gamma = (N, \mathcal{X}, \mathcal{U}, f, \{J^i\}_{i=1}^M, C).
\end{align}
For such a game, a GNE \cite{facchinei2010generalized} is attained at the set of feasible input sequences $\mathbf{u} = \{\mathbf{u}^{i}\}_{i=1}^M$ which minimize \eqref{eq:agent_cost} for all agents $i$. Formally, we define this solution concept as follows:
\begin{definition}
    A generalized Nash equilibrium (GNE) for the dynamic game $\Gamma$ is the set of open-loop solutions $\mathbf{u}^{\star} = \{\mathbf{u}^{i,\star}\}_{i=1}^M$ such that for each agent $i$:
    \begin{align}
        J^i(\mathbf{u}^{i,\star}, \mathbf{u}^{\neg i,\star}) \leq J^i(\mathbf{u}^{i}, \mathbf{u}^{\neg i,\star}), \ \forall \mathbf{u}^{i} \in \mathcal{U}^i(\mathbf{u}^{\neg i,\star}). \nonumber
    \end{align}
    If the condition holds only in some local neighborhood of $\mathbf{u}^{i,\star}$, then $\mathbf{u}^{\star}$ is denoted as  a local GNE.
\end{definition}
In other words, at a local GNE, agents cannot improve their cost by unilaterally perturbing their open-loop solution in a locally feasible direction. Furhtermore, it was shown in \cite[Theorem 2.2]{laine2023computation} that the local GNE for agent $i$ can be obtained equivalently by solving the following constrained finite horizon optimal control problems (FHOCP):
\begin{align} \label{eq:agent_fhocp}
    \mathbf{u}^{i,\star}(\mathbf{u}^{\neg i,\star}) = \arg\min_{\mathbf{u}^i} \ & \ J^i(\mathbf{u}^i, \mathbf{u}^{\neg i,\star}) \\
    \text{subject to} \ & \ C(\mathbf{u}^i, \mathbf{u}^{\neg i,\star}) \leq 0. \nonumber
\end{align}
where $\mathbf{u}^{\neg i,\star}$ correspond to local GNE solutions for the other agents. Note that we are assuming uniqueness of the local GNE of \eqref{eq:dynamic_game}. This assumption will be made formal in the next section. A distinct advantage of using \eqref{eq:agent_fhocp} to model agent interactions is that a dynamic game allows for a direct representation of agents with competing objectives as the $M$ objectives are considered separately instead of being summed together, which is typical in cooperative multi-agent approaches \cite{zhu2020trajectory}.

\section{An SQP Approach to Dynamic Games} \label{sec:sqp_approach}

In this section, we propose a method which iteratively solves for open-loop local GNE of dynamic games using sequential quadratic approximations. In particular, we will derive the algorithm and present guarantees on local convergence, which is based on established SQP theory \cite{boggs1995sequential}. We begin by defining the Lagrangian functions for the $M$ coupled FHOCPs in \eqref{eq:agent_fhocp}:
\begin{align*}
    \mathcal{L}^i(\mathbf{u}^i, \mathbf{u}^{\neg i,\star}, \lambda^i)  = J^i(\mathbf{u}^i, \mathbf{u}^{\neg i,\star}) +  C(\mathbf{u}^i, \mathbf{u}^{\neg i,\star})^\top \lambda^i,
\end{align*}
where we have again omitted the dependence on the initial state $x_0$ for brevity. As in \cite{lecleach2020algames}, we require that the Lagrange multipliers $\lambda^i \geq 0$ are equal over all agents, i.e $\lambda^i = \lambda^j = \lambda$, $\forall i,j \in \{1, \dots,M\}$. Since the multipliers reflect the sensitivity of the optimal cost w.r.t. constraint violation, this can be interpreted as a requirement for parity in terms of constraint violation for all agents. Under this condition, the GNE from \eqref{eq:agent_fhocp} are also known as normalized Nash equilibria \cite{rosen1965existence}.

A direct consequence of writing the constrained dynamic game in the coupled nonlinear optimization form of \eqref{eq:agent_fhocp} is that, subject to regularity conditions, solutions of \eqref{eq:agent_fhocp} must satisfy the KKT conditions below:
\begin{subequations} \label{eq:kkt}
\begin{align} 
    \nabla_{\mathbf{u}^i} \mathcal{L}^i(\mathbf{u}^{i,\star}, \mathbf{u}^{\neg i,\star}, \lambda^\star) &= 0, \ \forall i = 1, \dots, M, \label{eq:stationarity}\\
    C(\mathbf{u}^{1,\star},\dots, \mathbf{u}^{M,\star}) &\leq 0, \label{eq:primal_feasibility}\\
    C(\mathbf{u}^{1,\star},\dots, \mathbf{u}^{M,\star})^{\top}\lambda^\star &= 0, \\
    \lambda^\star & \geq 0.
\end{align}
\end{subequations}

We therefore propose to find a local GNE as a solution to the KKT system \eqref{eq:kkt} in an iterative fashion starting from an initial guess for the primal and dual solution, which we denote as $\mathbf{u}^i_0$ and $\lambda_0 \geq 0$ respectively, and taking steps $p_q^i$ and $p_q^\lambda$, at iteration $q$, to obtain the sequence of iterates:
\begin{align} \label{eq:sqp_step}
    \mathbf{u}^i_{q+1} = \mathbf{u}^i_q + p_q^i, \ \lambda_{q+1} = \lambda_q + p_q^\lambda.
\end{align}
In particular, we form a quadratic approximation of \eqref{eq:stationarity} and linearize the constraints in \eqref{eq:primal_feasibility} about the primal and dual solution at iteration $q$ in a SQP manner \cite{wright1999numerical} as follows:
\begin{align} \label{eq:sqp_approximation}
    L_q &= \begin{bmatrix}
    \nabla_{\mathbf{u}^1}^2 \mathcal{L}_q^1 & \nabla_{\mathbf{u}^2,\mathbf{u}^1} \mathcal{L}_q^1 & \dots & \nabla_{\mathbf{u}^M,\mathbf{u}^1} \mathcal{L}_q^1 \\
    \nabla_{\mathbf{u}^1,\mathbf{u}^2} \mathcal{L}_q^2 & \nabla_{\mathbf{u}^2}^2 \mathcal{L}_q^2 & \dots & \nabla_{\mathbf{u}^M,\mathbf{u}^2} \mathcal{L}_q^2 \\
    \vdots & \vdots & \ddots & \vdots \\
    \nabla_{\mathbf{u}^1,\mathbf{u}^M} \mathcal{L}_q^M & \nabla_{\mathbf{u}^2,\mathbf{u}^M} \mathcal{L}_q^M & \dots & \nabla_{\mathbf{u}^M}^2 \mathcal{L}_q^M
    \end{bmatrix}, \nonumber\\
    h_q &= \begin{bmatrix} \nabla_{\mathbf{u}^1} J_q^1 & \nabla_{\mathbf{u}^2} J_q^2 & \dots & \nabla_{\mathbf{u}^M} J_q^M \end{bmatrix}^\top, \nonumber\\
    G_q &= \begin{bmatrix} \nabla_{\mathbf{u}^1} C_q & \nabla_{\mathbf{u}^2} C_q & \dots & \nabla_{\mathbf{u}^M} C_q \end{bmatrix}, \nonumber\\
    B_q &= \text{proj}_{\succeq 0}((L_q+L_q^\top)/2) + \epsilon I,
\end{align}
where the subscript $q$ indicates that the corresponding quantity is evaluated at the primal and dual iterate $\mathbf{u}_q$ and $\lambda_q$. Here, $\epsilon \geq 0$ is a regularization coefficient, $I$ is the identity matrix of appropriate size, and 
\begin{align}
    \text{proj}_{\succeq 0}(X) = \arg \min_{Y\succeq 0} \|X-Y\|_F^2 \label{eq:psd_projection}
\end{align}
denotes the operation which projects the symmetric matrix $X \in \mathbb{S}^n$ onto the positive semi-definite cone, where $\|\cdot\|_F$ denotes the Frobenius norm. We note that the semidefinite program in \eqref{eq:psd_projection} admits the closed-form solution
$\text{proj}_{\succeq 0}(X) = \sum_{i=1}^n \max\{0, s_i\}v_iv_i^\top$, where $s_i$ and $v_i$ denote the $i$-th eigenvalue and eigenvector of $X$ respectively. 
Using the approximation in \eqref{eq:sqp_approximation}, we solve for the step in the primal variables via the following convex quadratic program (QP):
\begin{subequations} \label{eq:sqp_qp}
\begin{align} 
    p_q^\mathbf{u} = \arg \min_{p^1, \dots, p^M} \ & \ \frac{1}{2} {p^\mathbf{u}}^\top B_q {p^\mathbf{u}} + h_q^\top {p^\mathbf{u}} \label{eq:sqp_qp_cost}\\
    \text{subject to} \ & \ C_q + G_q {p^\mathbf{u}} \leq 0. \label{eq:sqp_qp_ineq_constraint}
\end{align}
\end{subequations}
where we denote $p^\mathbf{u} = [{p^1}^\top, \dots, {p^M}^\top]^\top$. Denote the Lagrange multipliers corresponding to the solution of \eqref{eq:sqp_qp} as $d_q$. We then define the step in the dual variables as 
\begin{align} \label{eq:dual_step}
    p_q^\lambda = d_q - \lambda_q,
\end{align}
which maintains nonnegativity of the dual iterates given $\lambda_0 \geq 0$. 
Finally, we note that in contrast to the approach used in \cite{spica2020real}, which requires the solution of $M$ optimization problems, our SQP procedure requires the solution of only a single QP at each iteration. 

\subsection{Local Behavior of Dynamic Game SQP}

We make the following assumptions about the primal and dual solutions of \eqref{eq:agent_fhocp}:
\begin{assumption} \label{asm:optimality}
    Solutions $\{\mathbf{u}^{i,\star}\}_{i=1}^M$ and $\lambda^\star$ of \eqref{eq:agent_fhocp} satisfy the following, for each $i\in\{1,\dots,M\}$:
    \begin{itemize}
        \item $\lambda^\star \perp C(\mathbf{u}^{i,\star}, \mathbf{u}^{\neg i,\star})$ and $\lambda_j^\star = 0 \iff C_j(\mathbf{u}^{i,\star}, \mathbf{u}^{\neg i,\star}) < 0$ for $j=1,\dots,n_c$.
        \item The rows of the Jacobian of the active constraints at the local GNE, i.e. $\nabla_{\mathbf{u}^i}\bar{C}(\mathbf{u}^{i,\star}, \mathbf{u}^{\neg i,\star})$, are linearly independent.
        \item $d^\top \nabla_{\mathbf{u}^i}^2 \mathcal{L}^i(\mathbf{u}^{i,\star}, \mathbf{u}^{\neg i,\star}, \lambda^\star) d > 0$, $\forall d \neq 0$ such that $\nabla_{\mathbf{u}^i}\bar{C}(\mathbf{u}^{i,\star}, \mathbf{u}^{\neg i,\star})^\top d = 0$.
    \end{itemize}
\end{assumption}
The first assumption is strict complementary slackness, the second is the linear independence constraint qualification (LICQ), and the third states that the Hessian of the Lagrangian function is positive definite on the null space of the active constraint Jacobians at the solution. 
It is straight forward to see that 
By standard optimization theory,
\eqref{eq:kkt} and Assumption~\ref{asm:optimality} together constitute necessary and sufficient conditions for a primal and dual solution of \eqref{eq:agent_fhocp} for agent $i$ to be locally optimal and unique given $\mathbf{u}^{\neg i,\star}$. When these conditions hold for the solutions over all agents, it was shown in \cite{laine2023computation} that satisfaction of the requirements for a unique local GNE follow immediately. Note that, as in \cite{boggs1995sequential} and \cite{laine2023computation}, Assumption~\ref{asm:optimality} is standard and can be verified a posteriori.

To analyze the local behavior of the iterative procedure as defined by \eqref{eq:sqp_step}, \eqref{eq:sqp_approximation}, and \eqref{eq:sqp_qp}, let us assume that $\mathbf{u}_0$ and $\lambda_0$ are close to the optimal solution and the subset of active constraints at the local GNE, which we denote as $\bar{C}$, with Jacobian $\bar{G}$, is known and constant at each iteration $q$. Therefore, for the purposes of this section, we can replace the inequality constraint in \eqref{eq:sqp_qp_ineq_constraint} with the equality constraint:
\begin{align} \label{eq:sqp_qp_eq_constraint}
    \bar{C}_q + \bar{G}_q p^{\mathbf{u}} = 0.
\end{align}
We refer to the QP constructed from \eqref{eq:sqp_qp_cost} and \eqref{eq:sqp_qp_eq_constraint} as EQP.

In the traditional derivation of the SQP procedure \cite{boggs1995sequential,wright1999numerical}, it was shown that under the aforementioned assumptions, the SQP step computed is identical to a Newton step for the corresponding KKT system. The SQP step therefore inherits the quadratic convergence rate of Newton's method in a local neighborhood of the optimal solution \cite[Theorem 3.1]{boggs1995sequential}. However, in the case of dynamic games, the equivalence between the SQP procedure and Newton's method is no longer exact since the matrix $L_q$ is not symmetric in general. To see this, let us first state the joint KKT system for the equality constrained version of \eqref{eq:agent_fhocp}:
\begin{align} \label{eq:root_finding_problem}
    & F(\mathbf{u}^{1,\star}, \dots, \mathbf{u}^{M,\star}, \lambda^\star) \\
    & \qquad = \begin{bmatrix}
        \nabla_{\mathbf{u}^1} \mathcal{L}^1(\mathbf{u}^{1,\star},\dots, \mathbf{u}^{M,\star}, \lambda^\star) \\
        \vdots \\
        \nabla_{\mathbf{u}^M} \mathcal{L}^M(\mathbf{u}^{1,\star},\dots, \mathbf{u}^{M,\star}, \lambda^\star) \\
        \bar{C}(\mathbf{u}^{1,\star},\dots, \mathbf{u}^{M,\star})
    \end{bmatrix} = 0, \nonumber
\end{align}
For the system of equations \eqref{eq:root_finding_problem}, the Newton step at iteration $q$ is the solution of the linear system:
\begin{align} \label{eq:newton_step}
    \begin{bmatrix}
    L_q & \bar{G}_q^\top \\
    \bar{G}_q & 0
    \end{bmatrix} \begin{bmatrix} \bar{p}_q^\mathbf{u} \\ \bar{p}_q^\lambda \end{bmatrix} = - \begin{bmatrix} h_q + \bar{G}_q^\top \lambda_q \\ \bar{C}_q \end{bmatrix}.
\end{align}
On the other hand, by the first order optimality conditions for EQP, we have that the SQP step must satisfy
\begin{align} \label{eq:sqp_kkt}
    \begin{bmatrix}
    B_q & \bar{G}_q^\top \\
    \bar{G}_q & 0
    \end{bmatrix} \begin{bmatrix} p_q^\mathbf{u} \\ p_q^\lambda \end{bmatrix} = - \begin{bmatrix} h_q + \bar{G}_q^\top \lambda_q \\ \bar{C}_q \end{bmatrix}.
\end{align}
When the matrix $L_q$ is positive definite and $\epsilon=0$, \eqref{eq:sqp_kkt} and \eqref{eq:newton_step} are equivalent. This corresponds to the special case of potential games \cite{zhu2008lagrangian}. However, this is not true in general for our SQP step, which implies that we cannot inherit the quadratic convergence of Newton's method. Instead, the SQP step from \eqref{eq:sqp_qp_cost} and \eqref{eq:sqp_qp_eq_constraint} can be seen as a symmetric approximation to the Newton step. As such, we establish guaranteed local linear convergence for our SQP procedure via established theory for SQP with approximate Hessians. Before proving the main result of this section, let us first define the bounded deterioration property, which essentially requires that the distance between a matrix and its approximations are bounded.
\begin{definition} \label{def:bounded_deterioration}
    A sequence of matrix approximations $\{B_q\}$ to $L^\star$ for the SQP method is said to have the property of \emph{bounded deterioration} if there exist constants $\alpha_1$ and $\alpha_2$ independent of $q$ such that:
    \begin{align*}
        \|B_{q+1}-L^\star\| \leq (1+\alpha_1 \sigma_q) \|B_q-L^\star\| + \alpha_2 \sigma_q,
    \end{align*}
    where $\sigma_q = \max(\|\mathbf{u}_{q+1}-\mathbf{u}^\star\|, \|\mathbf{u}_{q}-\mathbf{u}^\star\|, \|\lambda_{q+1}-\lambda^\star\|, \|\lambda_{q}-\lambda^\star\|)$, and $L^\star$ denotes the matrix $L$ in \eqref{eq:sqp_approximation} evaluated at the solution $\mathbf{u}^\star$, $\lambda^\star$.
\end{definition}
\begin{theorem}
Consider the dynamic game defined by \eqref{eq:dynamic_game}. Let Assumptions~\ref{asm:compact_set_differentiability} and \ref{asm:optimality} hold. Then there exist positive constants $\epsilon_1$ and $\epsilon_2$ such that if
\begin{align*}
    \|\mathbf{u}_0 - \mathbf{u}^\star\| \leq \epsilon_1, \ \|B_0 - L^\star\| \leq \epsilon_2,
\end{align*}
and $\lambda_0 = -(\bar{G}_0 \bar{G}_0^\top)^{-1}\bar{G}_0 h_0$, then the sequence $(\mathbf{u}_q, \lambda_q)$ generated by the SQP procedure \eqref{eq:sqp_step} and \eqref{eq:sqp_qp} converges linearly to $(\mathbf{u}^\star, \lambda^\star)$.
\end{theorem}
\begin{proof}
We obtain the result by showing that the conditions of \cite[Theorem 3.3]{boggs1995sequential} are satisfied. The first condition requires that the approximations $B_q$ are positive definite on the null space of $\bar{G}_q$. Since $B_q$ is constructed by projecting $L_q$ into the positive definite cone, this condition is satisfied trivially. 

The second condition requires that the sequence $\{B_q\}$ satisfies the bounded deterioration property from Definition~\ref{def:bounded_deterioration}. To show this, we begin with the following derivation:
\begin{align*}
    \|B_{q+1}-L^\star\| &= \|B_{q+1}-B_q + B_q - L^\star\| \\
    & \leq \|B_q - L^\star\| + \|B_{q+1}-B_q\|.
\end{align*}
From the above, it can be seen that $\{B_q\}$ satisfies the bounded deterioration property with $\alpha_1 = 0$ if there exists $\alpha_2$ such that $\|B_{q+1}-B_q\| \leq \alpha_2 \sigma_q$. We will show this via the mean value theorem, where we first observe that by Assumption~\ref{asm:compact_set_differentiability}, $L(\mathbf{u}_q,\lambda_q)$ is continuous and differentiable in $\mathbf{u}_q$ and $\lambda_q$. The smoothness of $B(\mathbf{u}_q,\lambda_q)$ is therefore dependent on the smoothness of the projection operator in \eqref{eq:psd_projection}. It was established in \cite{agrawal2019differentiating} that for parametric strongly convex conic optimization problems, there exists a unique, continuous, and differentiable mapping from the parameters to the optimal solution of the problem. The projection operator \eqref{eq:psd_projection} is a semi-definite program which admits a unique solution \cite{boyd2004convex}. Therefore, we may conclude that $B(\mathbf{u}_q,\lambda_q)$ is continuous and differentiable in $\mathbf{u}_q$ and $\lambda_q$. This allows us to establish the following inequalities:
\begin{align*}
    \|B_{q+1}-B_q\| &= \|B_{q+1}-B^\star+B^\star-B_q\| \\
    &\leq \|B_{q+1}-B^\star\| + \|B_q - B^\star\| \\
    &= \|B_{q+1}-B(\mathbf{u}^\star,\lambda_{q+1})+B(\mathbf{u}^\star,\lambda_{q+1})-B^\star\| \\
    &\quad + \|B_q-B(\mathbf{u}^\star,\lambda_{q})+B(\mathbf{u}^\star,\lambda_{q})-B^\star\| \\
    &\leq \|B_{q+1}-B(\mathbf{u}^\star,\lambda_{q+1})\| \\
    &\quad + \|B(\mathbf{u}^\star,\lambda_{q+1})-B^\star\| \\
    &\quad + \|B_q-B(\mathbf{u}^\star,\lambda_{q})\| \\
    &\quad + \|B(\mathbf{u}^\star,\lambda_{q})-B^\star\|.
\end{align*}
Since $B$ is continuous and differentiable, by the mean value theorem:
\begin{align*}
    &\|B_{q+1}-B(\mathbf{u}^\star,\lambda_{q+1})\| + \|B(\mathbf{u}^\star,\lambda_{q+1})-B^\star\| \\
    &\quad + \|B_q-B(\mathbf{u}^\star,\lambda_{q})\| + \|B(\mathbf{u}^\star,\lambda_{q})-B^\star\| \\
    \leq \ & \beta_1 \|\mathbf{u}_{q+1}-\mathbf{u}^\star\|+ \beta_2 \|\lambda_{q+1} - \lambda^\star\| \\
    &\quad + \beta_3 \|\mathbf{u}_{q}-\mathbf{u}^\star\| +\beta_4 \|\lambda_{q} - \lambda^\star\| \\
    \leq \ & \max(\beta_1, \beta_2, \beta_3, \beta_4) \sigma_q,
\end{align*}
where
\begin{align*}
    \beta_1 &= \sup_{t \in [0,1]} \|\nabla_\mathbf{u} B(t\mathbf{u}_{q+1}+(1-t)\mathbf{u}^\star, \lambda_{q+1})\|, \\
    \beta_2 &= \sup_{t \in [0,1]} \|\nabla_\mathbf{\lambda} B(\mathbf{u}^\star, t\lambda_{q+1}+(1-t)\lambda^\star)\|, \\
    \beta_3 &= \sup_{t \in [0,1]} \|\nabla_\mathbf{u} B(t\mathbf{u}_{q}+(1-t)\mathbf{u}^\star, \lambda_{q})\|, \\
    \beta_4 &= \sup_{t \in [0,1]} \|\nabla_\mathbf{\lambda} B(\mathbf{u}^\star, t\lambda_{q}+(1-t)\lambda^\star)\|.
\end{align*}
This satisfies the bounded deterioration property with $\alpha_2 = \max(\beta_1, \beta_2, \beta_3, \beta_4)$. Given Assumption~\ref{asm:optimality} holds, linear convergence of the SQP iterations follows directly as a consequence of \cite[Theorem 3.3]{boggs1995sequential}.
\end{proof}

\section{Numerical Robustness of Dynamic Game SQP} \label{sec:practical_considerations}

We have shown that our proposed SQP approach exhibits linear convergence when close to a local GNE. However, as is commonly seen with numerical methods for nonlinear optimization, a na{\"i}ve implementation of the procedure defined by \eqref{eq:sqp_step}, \eqref{eq:sqp_approximation}, and \eqref{eq:sqp_qp} often performs poorly due to overly aggressive steps leading to diverging iterates. In this section, we introduce multiple practical considerations which will help address this problem in practice. Namely, a merit function, a non-monotonic line search method, and a decaying regularization scheme. These components will be used to determine how much of the SQP step $p_q^\mathbf{u}$ and $p_q^\lambda$ can be taken to make progress towards a local GNE while remaining in a region about the current iterate where the QP approximation \eqref{eq:sqp_qp} is valid. We will show through examples that these additions are crucial in improving the numerical robustness of the proposed SQP approach.

\subsection{A Novel Merit Function}

Merit functions are commonly used in conjunction with a backtracking line search technique as a mechanism for measuring progress towards the optimal solution and limiting the step size taken by an iterative solver \cite{wright1999numerical}. This is important as the quality of the QP approximation \eqref{eq:sqp_qp} may be poor when far away from the iterate, which could result in divergence of the iterates if large steps are taken.

In traditional constrained optimization, merit functions typically track a combination of cost value and constraint violation. In the context of dynamic games, this is not as straightforward as the agents may have conflicting objectives and a proposed step may result in an increase in the objectives of some agents along with a decrease in others'. We also cannot just simply sum the objectives as minimizers of the combined cost function may not be local GNE. We therefore propose the following merit function:
\begin{align} \label{eq:merit_function}
    \phi(\mathbf{u},&\lambda, s;\mu) = \nonumber \\
    &\frac{1}{2} \| \underbrace{\begin{bmatrix} \nabla_{\mathbf{u}^1} \mathcal{L}^1(\mathbf{u}, \lambda) \\
        \vdots \\
        \nabla_{\mathbf{u}^M} \mathcal{L}^M(\mathbf{u}, \lambda) \end{bmatrix}}_{=\nabla \mathcal{L}(\mathbf{u},\lambda)} \|_2^2  + \mu \|C(\mathbf{u}) - s\|_1, 
\end{align}
and define $\gamma(\mathbf{u},\lambda) = (1/2)\|\nabla \mathcal{L}(\mathbf{u},\lambda)\|_2^2$. The slack variable $s = \text{min}(0, C(\mathbf{u}))$ is defined element-wise such that $C-s$ captures violation of the inequality constraints and we define the step $p^s = C(\mathbf{u}) + G(\mathbf{u})p^\mathbf{u} - s$. Compared to the merit function from \cite{laine2023computation}, which only included the first term of \eqref{eq:merit_function}, the novelty of our's is the $l^1$ norm term, whose purpose will be described shortly. Instead of measuring the agent objectives, our merit function tracks the first order optimality conditions in addition to constraint violation. It should be easy to see that the merit function attains a minimum of zero at any local GNE. However, we note that this merit function is not \emph{exact} \cite{wright1999numerical} since the first order conditions are only necessary for optimality.

Since we would like the sequence of iterates to converge to the minimizers of $\phi$, it follows that at each iteration we would like take a step in a descent direction of $\phi$. As such, we will now, in a manner similar to \cite{wright1999numerical}, analyze the directional derivative of the proposed merit function in the direction of the steps computed from \eqref{eq:sqp_qp} and describe a procedure for choosing the parameter $\mu$ and the corresponding \emph{conditions} such that a descent in $\phi$ is guaranteed.

Since the $l^1$ norm is not differentiable everywhere, we begin by taking a Taylor series expansion of $\gamma$ and $C-s$ for $\alpha \in (0,1]$:
\begin{align*}
    &\phi(\mathbf{u}+\alpha p^{\mathbf{u}},\lambda+\alpha p^\lambda, s+\alpha p^s;\mu) - \phi(\mathbf{u},\lambda, s;\mu) \\
    &\leq \alpha \nabla_{\mathbf{u},\lambda}\gamma\begin{bmatrix} p^{\mathbf{u}} \\ p^\lambda \end{bmatrix} + \mu\|C-s + \alpha(Gp^\mathbf{u}-p^s)\|_1 \nonumber \\
    & \qquad - \mu\|C-s\|_1 + \beta \alpha^2 \|p\|_2^2 \\
    & = \alpha\left(\nabla_{\mathbf{u},\lambda}\gamma\begin{bmatrix} p^{\mathbf{u}} \\ p^\lambda \end{bmatrix} - \mu\|C-s\|_1\right) + \beta \alpha^2 \|p\|_2^2,
\end{align*}
where the term $\beta \alpha^2 \|p\|_2^2$ bounds the second derivative terms for some $\beta > 0$. Following a similar logic, we can obtain the bound in the other direction:
\begin{align*}
    &\phi(\mathbf{u}+\alpha p^{\mathbf{u}},\lambda+\alpha p^\lambda, s+\alpha p^s;\mu) - \phi(\mathbf{u},\lambda, s;\mu) \\
    &\geq \alpha\left(\nabla_{\mathbf{u},\lambda}\gamma\begin{bmatrix} p^{\mathbf{u}} \\ p^\lambda \end{bmatrix} - \mu\|C-s\|_1\right) - \beta \alpha^2 \|p\|_2^2.
\end{align*}
Dividing the inequality chain by $\alpha$ and taking the limit $\alpha\to 0$, we obtain the directional derivative:
\begin{align} \label{eq:merit_function_directional_derivative}
    D(\phi(\mathbf{u},\lambda,s;\mu), p^{\mathbf{u}}, p^\lambda) = \nabla_{\mathbf{u},\lambda}\gamma\begin{bmatrix} p^{\mathbf{u}} \\ p^\lambda \end{bmatrix} - \mu\|C-s\|_1
\end{align}
From \eqref{eq:merit_function_directional_derivative}, it should be clear that given $C-s \neq 0$, there exists a value for $\mu > 0$ such that the directional derivative is negative. As such, we propose the following expression to compute $\mu$, given some $\rho \in (0, 1)$:
\begin{align} \label{eq:merit_parameter}
    \mu \geq \left(\nabla_{\mathbf{u},\lambda}\gamma\begin{bmatrix} p^{\mathbf{u}} \\ p^\lambda \end{bmatrix}\right)/((1-\rho)\|C-s\|_1),
\end{align}
which results in $D(\phi(\mathbf{u},\lambda;\mu), p^{\mathbf{u}}, p^\lambda) \leq -\rho\mu\|C-s\|_1$.

In the case when $C-s = 0$, we unfortunately cannot guarantee that the directional derivative will always be negative. To see this, let us analyze the first term of \eqref{eq:merit_function_directional_derivative}:
\begin{align} \label{eq:stationarity_norm_directonal_derivative}
    \nabla_{\mathbf{u},\lambda}\gamma\begin{bmatrix} p^{\mathbf{u}} \\ p^\lambda \end{bmatrix} &= \nabla \mathcal{L}^\top \begin{bmatrix} L & G^\top \end{bmatrix} \begin{bmatrix} p^{\mathbf{u}} \\ p^\lambda \end{bmatrix} \nonumber \\
        &= \nabla \mathcal{L}^\top \begin{bmatrix} B + R & G^\top \end{bmatrix} \begin{bmatrix} p^{\mathbf{u}} \\ p^\lambda \end{bmatrix} \nonumber \\
        &= -\nabla \mathcal{L}^\top \nabla \mathcal{L} + \nabla \mathcal{L}^\top R p^\mathbf{u},
\end{align}
where $R$ denotes the residual matrix i.e. $R = L-B$ and we arrive at the third equality by plugging in for the stationarity condition from \eqref{eq:sqp_qp}. From \eqref{eq:stationarity_norm_directonal_derivative}, it should be immediately apparent that when the residual matrix is large and dominates the first term, it is possible for the directional derivative to be positive. This can be interpreted as a condition on how well the positive definite $B$ actually approximates the original stacked Hessian matrix $L$. For dynamic games where the agents have highly coupled and differing objectives, i.e. when $L$ is highly non-symmetric, it is reasonable to expect that the approximation would suffer and that we may not be able to achieve a decrease in the merit function. For this reason, we utilize a non-monotone strategy for the line search step, which will be discussed in the following.

\subsection{A Non-Monotone Line Search Strategy}

Line search methods are used in conjunction with merit functions to achieve a compromise between the goals of making rapid progress towards the optimal solution and keeping the iterates from diverging. This is done by finding the largest step size $\alpha \in (0, 1]$ such that the following standard decrease condition is satisfied \cite{boggs1995sequential}:
\begin{align} \label{eq:merit_decrease_condition}
    \phi(&\mathbf{u}_q+\alpha p_q^\mathbf{u}, \lambda_q+\alpha p_q^\lambda, s_q+\alpha p_q^s;\mu) \\
    &\leq \phi(\mathbf{u}_q, \lambda_q, s_q;\mu) + \zeta \alpha D(\phi(\mathbf{u}_q,\lambda_q,s_q;\mu), p_q^{\mathbf{u}}, p_q^\lambda), \nonumber
\end{align}
where $\zeta \in (0,0.5)$. However, since our merit function is not exact, the line search procedure can be susceptible to poor local minima which do not correspond to local GNE. Moreover, there is evidence that requiring monotonic decrease in the merit function at each iteration may actually impede the solution process \cite{ferris1994nonmonotone}. We therefore include in our approach a non-monotone approach to line search called the \emph{watchdog} strategy \cite{conn2000trust}. Instead of insisting on sufficient decrease in the merit function at every iteration, this approach allows for relaxed steps to be taken for a certain number of iterations where increases in the merit function are allowed. The decrease requirement \eqref{eq:merit_decrease_condition} is then enforced after the prescribed number of relaxed iterations or if the step size exceeds some given threshold. The rationale behind this strategy is that we can use the relaxed steps as a way to escape regions where it is difficult to make progress w.r.t. the merit function. The algorithm we implement is based on the non-monotone stabilization scheme presented in \cite{dirkse1995path} and involves the taking of \emph{d-steps} and \emph{m-steps}, where \emph{d-steps} correspond to relaxed steps $(\alpha = 1)$ which disregard the merit decrease condition and \emph{m-steps} which enforce merit decrease through an Armijo backtracking linesearch on $\alpha$ in \eqref{eq:merit_decrease_condition}. Furthermore, the method stores the list of iterates which satisfy the merit decrease condition as checkpoints. This allows for a \emph{watchdog} step to reset the iterate to a checkpoint in the event that an \emph{m-step} was unsuccessful and to perform a backtracking linesearch from there.


\subsection{A Decaying Regularization Scheme}

The regularization parameter $\epsilon$ introduced in \eqref{eq:sqp_approximation} is a common mechanism used in iterative Newton-based procedures to control the size of $p^{\mathbf{u}}$. However, the selection of the regularization parameter is non-trivial in many cases as the convergence behavior of the iterative procedure can be highly sensitive to the value of $\epsilon$. Often, it is observed that iterates can diverge if $\epsilon$ is too small, whereas progress is slow if $\epsilon$ is too big. To address this challenge, one method which has seen success is that of decaying regularization which given some initial setting of $\epsilon$ gradually decreases its value over iterations of the algorithm \cite{lecleach2020algames,wachter2006implementation,mehr2023maximum}. In this work, we take a similar approach to updating the values of $\epsilon$, but importantly, integrate it into the non-monotone line search procedure. Specifically, for some initial regularization value $\epsilon_0 \geq 0$ and decay multiplier $\eta \in (0, 1]$, we impose the following update rule:
\begin{align}
    \epsilon_{q+1} = \begin{cases}
    \epsilon_q & \qquad \text{if \emph{d-step}}, \\
    \eta \epsilon_q & \qquad \text{if \emph{m-step}},
    \end{cases}
\end{align}
where the regularization value only decays after a step which satisfies the merit function decrease condition is found. The reason for this is to avoid reducing the regularization strength when potentially spurious relaxed \emph{d-steps} are taken and the iterate is reset to a checkpoint. 

\section{The Dynamic Game SQP Algorithm}
\begin{algorithm}[t!]
    \SetAlgoLined
	\KwIn{$\mathbf{u}_0$, $\epsilon_0$}
	$q \leftarrow 0$\;
	$\mathbf{u}_q \leftarrow \mathbf{u}_0$, $\lambda_q \leftarrow \max(0, -(G_0 G_0^\top)^{-1}G_0 h_0)$\;
	\While{not converged}{
	    $B_q(\epsilon_q), \ h_q, \ G_q, \ C_q \leftarrow$ \eqref{eq:sqp_approximation}\;
	    $p_q^\mathbf{u}, \ p_q^\lambda \leftarrow$ \eqref{eq:sqp_qp}, \eqref{eq:dual_step}\;
        $s_q \leftarrow \min(0, C_q)$, $p_{q}^s \leftarrow C_q + G_q p_q^\mathbf{u} - s_q$\;
	    \eIf{$C_q-s_q\neq 0$}{
            Compute $\mu$ from \eqref{eq:merit_parameter}\;
	    }{
	        $\mu \leftarrow 0$\;
	    }
	    $\mathbf{u}_{q+1}, \ \lambda_{q+1}, \epsilon_{q+1} \leftarrow$ watchdog line search\;
	    $q \leftarrow q + 1$
	}
	\Return $\mathbf{u}^\star \leftarrow \mathbf{u}_q$, $\lambda^\star \leftarrow \lambda_q$\;
	\caption{Dynamic Game SQP (DG-SQP)}
	\label{alg:dynamic_game_sqp}
\end{algorithm}

By combining the elements discussed in the previous section, we arrive at the dynamic game SQP (DG-SQP) algorithm, which is presented in Algorithm~\ref{alg:dynamic_game_sqp}. The algorithm requires as input initial guesses of open-loop input sequences for each agent and the initial regularization value $\epsilon_0$. Line 2 initializes the primal and dual iterates, where the dual variables are initialized as the least squares solution to \eqref{eq:stationarity}. Lines 3 to 14 perform the SQP iteration which has been described in Sections~\ref{sec:sqp_approach} and \ref{sec:practical_considerations}. An iterate is said to have converged to a local GNE if it satisfies the KKT conditions described in \eqref{eq:kkt} up to some user specified tolerance. Namely, for some given $\rho_1, \ \rho_2, \ \rho_3 > 0$, we require the conditions $\|\nabla\mathcal{L}(\mathbf{u}_q, \lambda_q)\|_\infty \leq \rho_1, \ \|C(\mathbf{u}_q)\|_\infty \leq \rho_2, \ |\lambda_q^\top C(\mathbf{u}_q)| \leq \rho_3$ be satisfied in order for the algorithm to terminate successfully. 
The algorithm outputs the open-loop strategies for the $M$ agents and the corresponding Lagrange multipliers. 

\section{Frenet-Frame Games for Autonomous Racing} \label{sec:approximation}

\begin{figure}[t] 
    \centering
    \begin{subfigure}[t]{0.45\columnwidth}
        \centering
        \includegraphics[width=\columnwidth]{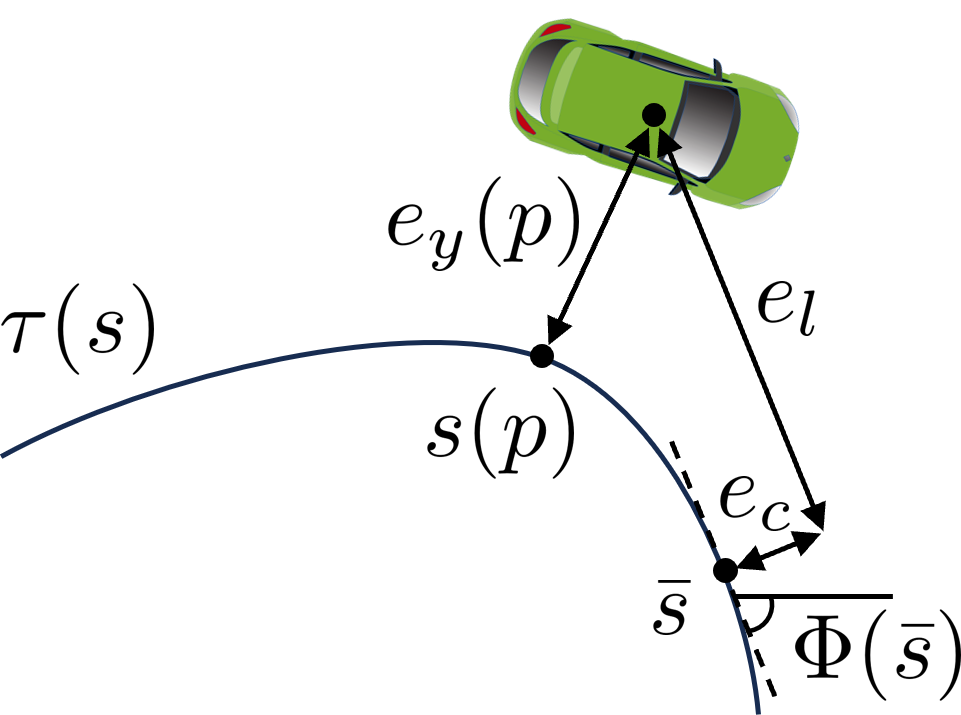}
        \caption{}
        \label{fig:contouring_lag_errors}
    \end{subfigure}
    ~
    \begin{subfigure}[t]{0.45\columnwidth}
        \centering
        \includegraphics[width=\columnwidth]{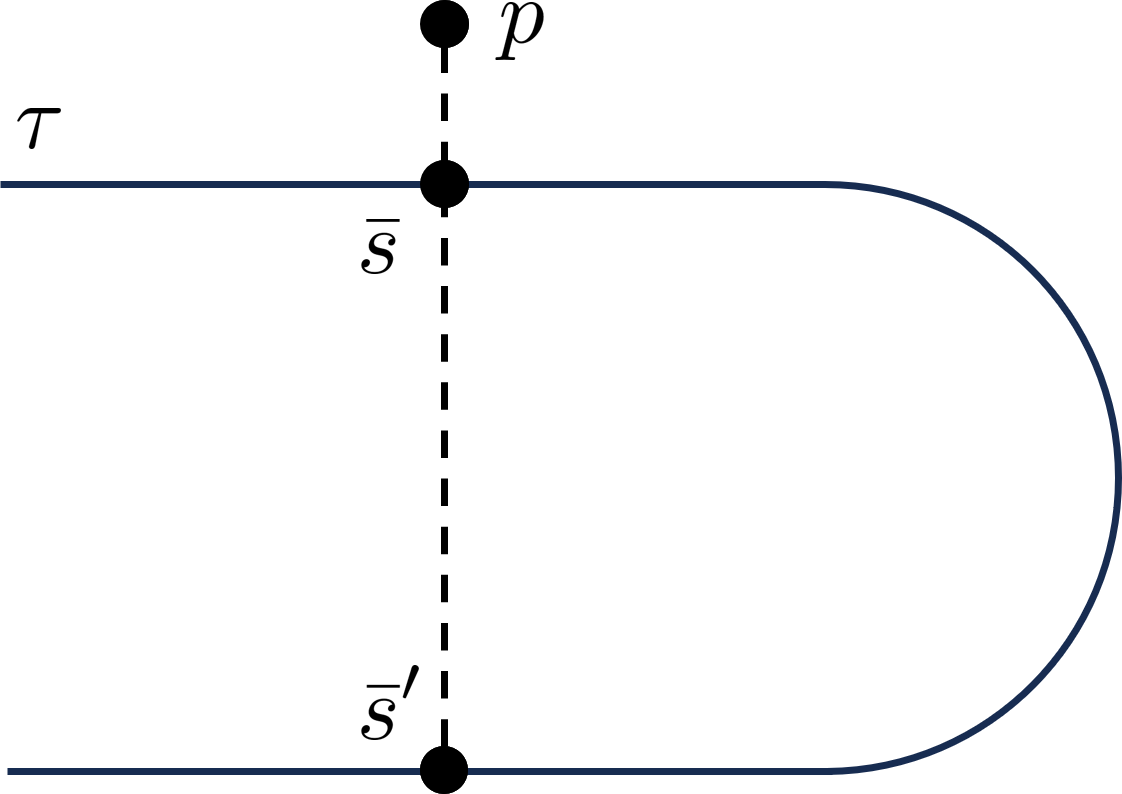}
        \caption{}
        \label{fig:poor_approx}
    \end{subfigure}
    \caption{(a) Illustration of the contouring and lag errors $e_c$ and $e_l$. (b) An example of a poor approximation $\bar{s}'$ where $e_l(p,\bar{s})=e_l(p,\bar{s}')=0$, but $\bar{s}' \neq s(p)$.}
\end{figure}

This section presents the second main contribution of this paper. 
We formulate the autonomous racing problem as a dynamic game using ideas from contouring control \cite{liniger2015optimization,faulwasser2009model} and leverage the solver presented in the previous section to compute the solution to such game.
Numerical studies in Section~\ref{sec:results} will show that with this formulation, we achieve significantly higher success rates when solving for GNE of racing dynamic games. 


We consider the class of racing dynamic games where 
the race track is represented as
a parametric path $\tau:[0, L] \mapsto \mathbb{R}^2$, which is assumed to be twice differentiable and of total length $L$. In particular, when given an arclength $s \in [0, L]$, $\tau$ returns the inertial $x$-$y$ position of the path and we may write the path tangent angle at $s$ is $\Phi(s) = \arctan(\tau_y'(s)/\tau_x'(s))$. In the case where $\tau$ forms a closed circuit we have that $\tau(0) = \tau(L)$, $\tau'(0) = \tau'(L)$, $\tau''(0) = \tau''(L)$, where $\tau'$ and $\tau''$ are the element-wise first and second derivatives of $\tau$.
Given a path $\tau$, we may then transform an inertial frame position $p = (x,y)$ and yaw angle $\psi$ into a path-relative pose as follows: 
\begin{subequations} \label{eq:frenet_pose}
    \begin{align}
        s(p) &= \arg \min_s \|\tau(s) - p\|_2, \label{eq:frenet_s} \\
        e_y(p) &= [-\sin \Phi(s(p)), \cos \Phi(s(p))] (p - \tau(s(p)), \label{eq:frenet_ey} \\
        e_\psi(p,\psi) &= \psi - \Phi(s(p)), \nonumber
    \end{align}
\end{subequations}
which are the path progress, the lateral displacement from the path, and the heading deviation from the path tangent corresponding to $p$ and $\psi$ respectively. 
These quantities, which are illustrated in Figure~\ref{fig:contouring_lag_errors}, are known as the coordinates of a Frenet reference frame \cite{micaelli1993trajectory} and they enable a straightforward expression of quantities which are important to the formulation of the autonomous racing problem as a dynamic game \cite{spica2020real}. First, we may express the difference in track progress between two agents as $s(p^i) - s(p^j)$. Maximization of this quantity in the game objective $J^i$ \eqref{eq:agent_cost_u} captures the direct competition between agents $i$ and $j$ in racing. Furthermore, using the lateral deviation in \eqref{eq:frenet_ey}, we may easily express track boundary constraints for agent $i$ as simple bounds on $e_y^i(p^i)$, which are then included in the game constraints $C$ \eqref{eq:joint_constraints}. Repeated evaluation of the potentially nonlinear optimization problem (NLP) in \eqref{eq:frenet_s} can be avoided by directly describing the evolution of a vehicle's kinematics in the Frenet reference frame \cite{micaelli1993trajectory}. In this work, we consider two vehicle models which can be expressed this way. The first is the Frenet-frame kinematic bicycle model which, for vehicle $i$, has the state and input vector:
\begin{subequations}
    \begin{align}
        x_{\text{kin}}^i &= [v, s, e_y, e_\psi]^\top \in \mathcal{X}_{\text{kin}}^i, \\
        u_{\text{kin}}^i &= [a, \delta_f]^\top \in \mathcal{U}_{\text{kin}}^i,
    \end{align}
\end{subequations}
where $v$ and $a$ are the velocity and acceleration of vehicle $i$ along its direction of travel and $\delta_f$ is the steering angle of the front wheel. The second is the Frenet-frame dynamic bicycle model which, for vehicle $i$, has the state and input vector:
\begin{subequations}
    \begin{align}
        x_{\text{dyn}}^i &= [v_x, v_y, \omega, s, e_y, e_\psi]^\top \in \mathcal{X}_{\text{dyn}}^i, \\
        u_{\text{dyn}}^i &= [a_x, \delta_f]^\top \in \mathcal{U}_{\text{dyn}}^i,
    \end{align}
\end{subequations}
where $v_x$ and $v_y$ are the longitudinal and lateral velocity, $\omega$ is the yaw rate, and $a_x$ is the longitudinal acceleration of vehicle $i$. The following functions describe the state evolution of the kinematic or dynamic bicycle models and expressions for both can be found in the Appendix. 
\begin{align}
    x_{\text{kin}}^{i,+} &= f_{\text{kin}}^i(x_{\text{kin}}^i, u_{\text{kin}}^i) \label{eq:frenet_kin_bike} \\
    x_{\text{dyn}}^{i,+} &= f_{\text{dyn}}^i(x_{\text{dyn}}^i, u_{\text{dyn}}^i) \label{eq:frenet_dyn_bike}
\end{align}
These can then be concatenated to construct the game dynamics in \eqref{eq:joint_dynamics}:
\begin{align}
    f_* = \text{vec}(f_*^1, \dots, f_*^M), \label{eq:frenet_joint_dynamics}
\end{align}
where the subscript $*\in\{\text{kin}, \text{dyn}\}$ indicates the vehicle model.
A Frenet-frame dynamic game with either the kinematic or dynamic bicycle model can therefore be defined as:
\begin{align} \label{eq:frenet_dynamic_game}
    \Gamma_* = (N, \mathcal{X}_*,  \mathcal{U}_*,  f_*, \{J_*^i\}_{i=1}^M,  C_*).
\end{align}

Direct expression of the vehicle kinematics in a Frenet reference frame offers two benefits in the formulation of racing tasks. The first is that it allows for the expression of progress maximization as a linear function of the state variable $s$. The second is that it allows for track boundary constraints to be imposed as simple bounds on the state variable $e_y$ \cite{rosolia2019learning,jung2023autonomous,zhu2023gaussian}. 
However, this comes at the cost of introducing singularities at the centers of curvature (as seen in the numerator of \eqref{eq:s_dot}) which can cause numerical instabilities especially on tracks with tight turns. Furthermore, Frenet-frame kinematics also complicate the expression of obstacle avoidance constraints in multi-agent settings, which will be seen in our description of the racing scenarios in Section~\ref{sec:results}. Our formulation, which applies ideas from contouring control \cite{liniger2015optimization,faulwasser2009model,lam2010model} in the context of dynamic games, approximates the evolution of $s$ and $e_y$ without introducing singularities, and results in a dynamic game with inertial kinematics, where the approximations of $s$ and $e_y$ can be leveraged in a similar manner as before for simple description of racing tasks. 

Let us first define the inertial-frame counterparts of the kinematic and dynamic bicycle models for vehicle $i$, which have as state and input vectors 
\begin{subequations}
    \begin{align}
        \bar{x}_{\text{kin}}^i &= [v, x, y, \psi]^\top, \in \bar{\mathcal{X}}_{\text{kin}}^i \\
        \bar{u}_{\text{kin}}^i &= [a, \delta_f]^\top \in \bar{\mathcal{U}}_{\text{kin}}^i.
    \end{align}
\end{subequations}
and
\begin{subequations}
    \begin{align}
        \bar{x}_{\text{dyn}}^i &= [v_x, v_y, \omega, x, y, \psi]^\top \in \bar{\mathcal{X}}_{\text{dyn}}^i, \\
        \bar{u}_{\text{dyn}}^i &= [a_x, \delta_f]^\top \in \bar{\mathcal{U}}_{\text{dyn}}^i,
    \end{align}
\end{subequations}
respectively. The dynamics functions describe the state evolution of the inertial-frame kinematic or dynamic bicycle models and expressions for both are again included in the Appendix. 
\begin{align}
    \bar{x}_{\text{kin}}^{i,+} &= \bar{f}_{\text{kin}}^i(\bar{x}_{\text{kin}}^i, \bar{u}_{\text{kin}}^i) \label{eq:inertial_kin_bike} \\
    \bar{x}_{\text{dyn}}^{i,+} &= \bar{f}_{\text{dyn}}^i(\bar{x}_{\text{dyn}}^i, \bar{u}_{\text{dyn}}^i) \label{eq:inertial_dyn_bike}
\end{align}
We begin the derivation of our approximation by augmenting the state and input vectors for vehicle $i$ with the variables $\bar{s}^i$ and $\bar{v}^i$ to obtain 
\begin{align}
    \hat{x}_*^i = (\bar{x}_*^i, \bar{s}^i) \in \hat{\mathcal{X}}_*^i, \ \hat{u}_*^i = (\bar{u}_*^i, \bar{v}^i) \in \hat{\mathcal{U}}_*^i, \label{eq:aug_state_input}
\end{align}
where we use the subscript $*\in\{\text{kin}, \text{dyn}\}$ to indicate that this augmentation may be performed on the state and input vectors for either the kinematic and dynamic bicycle models. The variable $\bar{s}^i$ represents our approximation of the progress $s$ for vehicle $i$ along a given path $\tau$ and evolves according to simple integrator dynamics: 
\begin{align}
    \bar{s}_{k+1}^i = \bar{s}_{k}^i + \Delta t \cdot \bar{v}_k^i, \label{eq:s_approx_dyanmics}
\end{align} 
with initial condition $\bar{s}_0^i = s(p_0^i)$, where $\bar{v}^i$ is a decision variable that can be thought of as an approximate arcspeed for agent $i$. We denote the augmented dynamics function as 
\begin{align}
    \hat{x}_*^{i,+} = \hat{f}_*^i(\hat{x}_*^i, \hat{u}_*^i), \label{eq:augmented_agent_dynamics}
\end{align}
which is defined as the concatenation of the original inertial-frame vehicle dynamics $\bar{f}_*^i$ in \eqref{eq:inertial_kin_bike} or \eqref{eq:inertial_dyn_bike} with \eqref{eq:s_approx_dyanmics}. Construction of the game dynamics $\hat{f}_*$ can then be done by simply concatenating the augmented dynamics functions $\hat{f}_*^i$ in \eqref{eq:augmented_agent_dynamics} for all vehicles:
\begin{align}
    \hat{f}_* = \text{vec}(\hat{f}_*^1, \dots, \hat{f}_*^M). \label{eq:augmented_joint_dynamics}
\end{align}
Now, given an inertial position $p^i$ and approximate progress $\bar{s}^i$, we define the lag error for vehicle $i$ as follows:
\begin{align}
    e_l(p^i, \bar{s}^i) &= [-\cos \Phi(\bar{s}^i), -\sin \Phi(\bar{s}^i)] (p^i - \tau(\bar{s}^i)). \label{eq:lag_error}
\end{align}
Specifically, the lag error approximates the difference between $\bar{s}^i$ and $s(p^i)$. We note that the complement of the lag error can be written as $e_c(p^i, \bar{s}^i) = [-\sin \Phi(\bar{s}^i), \cos \Phi(\bar{s}^i)] (p^i - \tau(\bar{s}^i))$. This term is known as the contouring error \cite{lam2010model} and it approximates the lateral deviation of the vehicle from the path at $\bar{s}^i$. It is straightforward to see through the illustration in Figure~\ref{fig:contouring_lag_errors} that when $\bar{s}$ is in a local neighborhood of $s(p)$ then $e_l \approx 0$ implies that $\bar{s} \approx s(p)$ and $e_c \approx e_y(p)$. However, if $\bar{s}$ is far from $s(p)$, as in the case of Figure~\ref{fig:poor_approx}, then it is entirely possible that the difference between $\bar{s}$ and $s(p)$ is arbitrarily large despite $e_l = 0$. Therefore, in order to for $\bar{s}$ to be a good approximation, we must not only drive $e_l$ to zero, but we must also attempt to keep $\bar{s}$ close to $s(p)$. This leads us to make the following two modifications to the game objectives and constraints.

First, to minimize the lag error, we introduce it into the augmented game objective for agent $i$ as follows: 
\begin{align} 
    \hat{J}_*^i(\hat{\mathbf{u}}_*^i, \hat{\mathbf{u}}_*^{\neg i}) = J_*^i(\hat{\mathbf{u}}_*^i, \hat{\mathbf{u}}_*^{\neg i}) + \sum_{k=1}^{N} q_l e_l(p_k^i(\bar{\mathbf{u}}_*^i), \bar{s}_k^i(\bar{\mathbf{v}}^i))^2, \label{eq:agent_cost_u_aug}
\end{align}
where $J^i_*$ is the original game objective for agent $i$ in \eqref{eq:frenet_dynamic_game}, $q_l \gg 0$ is a weight on the lag error cost, and $p_k^i(\cdot)$ and $\bar{s}_k^i(\cdot)$ denote functions which, given the input sequences $\bar{\mathbf{u}}^i$ and $\bar{\mathbf{v}}^i$, roll out the dynamics function $\hat{f}^i_*$ up to time step $k$ and return the $x$-$y$ position and approximate progress respectively. For brevity, we have omitted the explicit dependence of $\hat{J}^i_*$ on $\bar{s}_0^i$: the path progress corresponding to the state $\hat{x}_{*,0}^i$. This can be easily computed via \eqref{eq:frenet_s} as $\bar{s}_0^i = s(p_0^i)$. Furthermore, under a slight abuse of notation, we replace the arguments of the original Frenet-frame game objectives in \eqref{eq:frenet_dynamic_game} (i.e. the first term in \eqref{eq:agent_cost_u_aug}) with the augmented input sequences to reflect that any terms in $J^i_*$ which originally depended on $s^i$ and $s^{\neg i}$ should be replaced with the approximations $\bar{s}^i$ and $\bar{s}^{\neg i}$. Second, to keep $\bar{s}_k^i$ close to $s(p_k^i)$, we impose the simple bounds on the approximate arc speed input: $|\bar{v}_k^i| \leq v_{\text{max}}$. This helps in avoiding undesirable local minima of the lag error term in $\eqref{eq:agent_cost_u_aug}$ which correspond to poor approximations with $\bar{s}_k^i$ as shown in Figure~\ref{fig:poor_approx}. We denote the concatenation of the arcspeed bounds for all agents $i \in \{1, \dots, M\}$ and time steps $k \in \{1, \dots, N\}$ as $C_{v}(\bar{\mathbf{v}}^1, \dots, \bar{\mathbf{v}}^M)$.

Next, we turn to the track boundary constraints which can no longer be expressed as simple bounds on $e_y$. Instead, we may conveniently use the contouring error $e_c$ in Figure~\ref{fig:contouring_lag_errors} as an approximation to $e_y$ to define the following constraints for agent $i$. 
\begin{align}
    -w^-(\bar{s}_k^i(\bar{\mathbf{v}}^i)) \leq e_c(p_k^i(\bar{\mathbf{u}}_*^i),\bar{s}_k^i(\bar{\mathbf{v}}^i)) \leq w^+(\bar{s}_k^i(\bar{\mathbf{v}}^i)), \label{eq:contouring_track_boundary}
\end{align}
where $w^+, w^- : [0, L] \mapsto \mathbb{R}_{++}$ are functions which, for a value of $s$, return the distance from $\tau(s)$ to the track boundaries in the directions normal to $\tau$ at $s$. 

Finally, we define the augmented constraint function $\hat{C}_*$ as the concatenation of the original constraints of the dynamic game $C_*$ in \eqref{eq:frenet_dynamic_game} with the arc speed bounds:
\begin{align}
    \hat{C}_*(\hat{\mathbf{u}}^1, \dots, \hat{\mathbf{u}}^M) = \text{vec}(C_*(\hat{\mathbf{u}}^1, \dots, \hat{\mathbf{u}}^M), C_v(\bar{\mathbf{v}}^1, \dots, \bar{\mathbf{v}}^M)), \label{eq:augmented_constraints}
\end{align} 
where through a similar abuse of notation as before, we replace the arguments of the original Frenet-frame game constraints $C_*$ with the augmented input sequences to reflect that the original track boundary constraints in $e_y$ should be replaced with \eqref{eq:contouring_track_boundary}. Combining the elements discussed above, we define the following dynamic game:
\begin{align} \label{eq:approx_dynamic_game}
    \hat{\Gamma}_* = (N, \hat{\mathcal{X}}_*,  \hat{\mathcal{U}}_*,  \hat{f}_*, \{ \hat{J}_*^i\}_{i=1}^M,  \hat{C}_*),
\end{align}
which can be viewed as an approximation to the racing dynamic game $\Gamma_*$ in \eqref{eq:frenet_dynamic_game}, which uses the exact Frenet-frame kinematics when defining the game dynamics, agent objective, and constraint functions. Table~\ref{tab:dynamic_game_summary} summarizes the components of each of the exact and approximate dynamic games which will be solved in the following numerical studies.

\begin{table}[t]
\centering
\caption{Summary of Exact and Approximate Dynamic Games}
\begin{tabular}{c|c|c|c} 
& Dynamics & Objectives & Constraints \\
\midrule[1pt]
$\Gamma_{\text{kin}}$ & $f_\text{kin}$  \eqref{eq:frenet_kin_bike},\eqref{eq:frenet_joint_dynamics} & $\{J_\text{kin}\}_{i=1}^M$ \eqref{eq:agent_cost_u} & $C_\text{kin}$ \eqref{eq:joint_constraints} \\
$\Gamma_{\text{dyn}}$ & $f_\text{dyn}$  \eqref{eq:frenet_dyn_bike},\eqref{eq:frenet_joint_dynamics} & $\{J_\text{dyn}\}_{i=1}^M$ \eqref{eq:agent_cost_u} & $C_\text{dyn}$ \eqref{eq:joint_constraints} \\
$\hat{\Gamma}_{\text{kin}}$ & $\hat{f}_\text{kin}$ \eqref{eq:inertial_kin_bike},\eqref{eq:augmented_agent_dynamics},\eqref{eq:augmented_joint_dynamics} & $\{\hat{J}_\text{kin}\}_{i=1}^M$ \eqref{eq:agent_cost_u},\eqref{eq:agent_cost_u_aug} & $\hat{C}_\text{kin} \eqref{eq:joint_constraints},\eqref{eq:augmented_constraints}$ \\
$\hat{\Gamma}_{\text{dyn}}$ & $\hat{f}_\text{dyn}$ \eqref{eq:inertial_dyn_bike},\eqref{eq:augmented_agent_dynamics},\eqref{eq:augmented_joint_dynamics} & $\{\hat{J}_\text{dyn}\}_{i=1}^M$ \eqref{eq:agent_cost_u},\eqref{eq:agent_cost_u_aug} & $\hat{C}_\text{dyn} \eqref{eq:joint_constraints},\eqref{eq:augmented_constraints}$ \\
\end{tabular}
\label{tab:dynamic_game_summary}
\end{table}

\section{Numerical Study} \label{sec:results}

In the following numerical studies, we examine the performance of our DG-SQP solver in solving for \emph{open-loop} GNE of three head-to-head racing scenarios, which are defined in Part A. In Part B, we examine the effect of the proposed merit function and non-monotone line search strategy on the success rate of the DG-SQP solver via an ablation study on these components. In Part C, we investigate the effect of the proposed decaying regularization scheme through a sensitivity study on the regularization value and decay rate. In Part D, we compare the performance of our DG-SQP solver against the PATH solver from \cite{dirkse1995path} and show improved performance when solving both the exact and approximate dynamic games $\Gamma_*$ and $\hat{\Gamma}_*$. 
Our DG-SQP solver was implemented in Python and for the PATH solver, we use the Julia implementation \texttt{PATHSolver.jl} for comparison, but call it through a custom Python wrapper in order for us to easily pass identical definitions of the following dynamic games to it. Our implementation can be found at \texttt{https://github.com/zhu-edward/DGSQP}.

\subsection{Scenario Description}

We perform our numerical studies on three head-to-head racing scenarios which were constructed to showcase the performance of our DG-SQP solver with different vehicle models and track geometries.

\begin{figure}[t] 
    \centering
    \includegraphics[width=0.99\columnwidth]{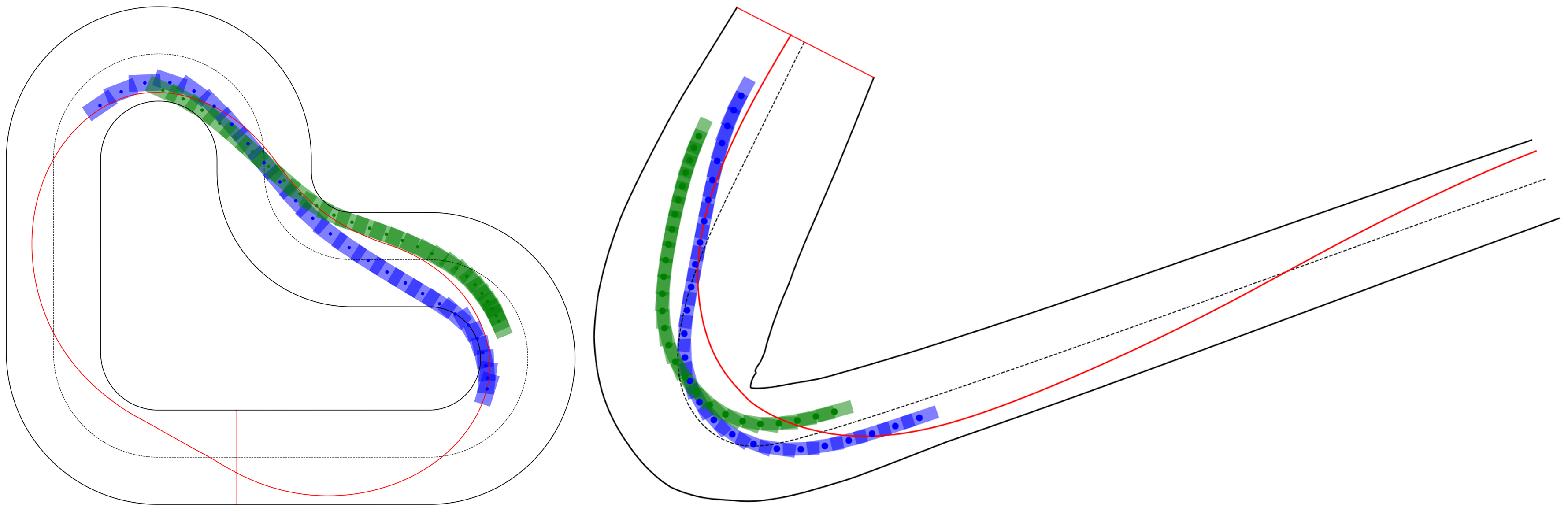}
    \caption{Example head-to-head racing open-loop GNE solutions with $N=25$ for the L-shaped track and 1/10th scale RC car (left) and the first hairpin turn of the Austin F1 track and a full scale race car (right). The red trajectories show pre-computed race lines which are used to warm start Scenarios 2 and 3 of our numerical study. In both plots, the cars are traveling in the counter-clockwise direction.}
    \label{fig:exemplary_gnes}
\end{figure}

\subsubsection{Scenario 1}

In this scenario, we use the L-shaped track shown in the left plot of Figure~\ref{fig:exemplary_gnes}, where the vehicles travel in the counter clockwise direction from the red start line. We model both vehicles using the kinematic bicycle model $f_{\text{kin}}^i$ with parameters that correspond to a 1/10th scale RC car. The discrete time dynamics are obtained via 4th order Runge-Kutta discretization with a time step of $\Delta t = 0.1s$. The Frenet-frame, or exact, dynamic game $\Gamma_{\text{kin}}$ is defined with the game dynamics $f_{\text{kin}} = \text{vec}(f_{\text{kin}}^1, f_{\text{kin}}^2)$ and objectives:
\begin{align*}
    J_{\text{kin}}^1(\mathbf{u}^1, \mathbf{u}^2) &= \sum_{k=0}^{N-1} \|u_k^1\|_{R^1}^2 + \|\Delta u_k^1\|_{P^1}^2 \\
    &\qquad + q_2^1 s_N^2(\mathbf{u}^2) -q_1^1 s_N^1(\mathbf{u}^1) \\
    J_{\text{kin}}^2(\mathbf{u}^2, \mathbf{u}^1) &= \sum_{k=0}^{N-1} \|u_k^2\|_{R^2}^2 + \|\Delta u_k^2\|_{P^2}^2 \\
    &\qquad + q_1^2 s_N^1(\mathbf{u}^1) -q_2^2 s_N^2(\mathbf{u}^2)
\end{align*}
where $\Delta u_k^i = u_k^i - u_{k-1}^i$, $u_{-1}^i$ is the input from the previous time step, $R^i, P^i \succeq 0$ are weights on the quadratic input and input rate penalties, $q_1^i, q_2^i > 0$ are weights on the competition costs, and $\|x\|_A^2 = x^\top A x$. The constraints $C_{\text{kin}}$ are defined as follows: 
\begin{subequations}
    \begin{align}
        u_l \leq u_k^i \leq u_u, \ i = 1, 2 \label{eq:input_limits} \\
        -w/2 \leq e_{y,k}^i(\mathbf{u}^i) \leq w/2, \ i = 1, 2 \label{eq:track_boundaries} \\
        \|p^1_k(\mathbf{u}^1) - p^2_k(\mathbf{u}^2)\|_2^2 \geq (r^1+r^2)^2, \label{eq:collision_avoidance}
    \end{align}
\end{subequations}
which consist of input limits \eqref{eq:input_limits}, track boundary constraints \eqref{eq:track_boundaries}, and collision avoidance constraints \eqref{eq:collision_avoidance}. In the above, $u_l$ and $u_u$ are the lower and upper bounds on the input, $w$ is the track width, and $r^1,r^2$ are the radii of the circular collision buffers for the two vehicles. When implementing the collision avoidance constraint for the exact dynamic game $\Gamma_\text{kin}$, the mapping $p^i_k(\mathbf{u}^i)$, which returns the inertial $x$-$y$ position of agent $i$ at time step $k$, is written as:
\begin{align*}
    p^i_k(\mathbf{u}^i) = \tau(s^i_k(\mathbf{u}^i)) + \frac{e_{y,k}^i(\mathbf{u}^i)}{\|\tau'(s^i_k(\mathbf{u}^i))\|_2} \begin{bmatrix}
        -\tau_y'(s^i_k(\mathbf{u}^i)) \\
        \tau_x'(s^i_k(\mathbf{u}^i))
    \end{bmatrix}.
\end{align*} 
For the approximate dynamic game $\hat{\Gamma}_\text{kin}$, we construct the game dynamics as $\hat{f}_{\text{kin}} = \text{vec}(\hat{f}_{\text{kin}}^1, \hat{f}_{\text{kin}}^2)$. The game objectives $J_\text{kin}^i$ are obtained by replacing any occurrences of $s_N^i(\cdot)$ with $\bar{s}_N^i(\cdot)$ in the expressions for $J_{\text{kin}}^i$ above. The game constraints are constructed as $\hat{C}_\text{kin} = \text{vec}(C_\text{kin}, C_v)$ where $e_{y,k}^i(\mathbf{u}^i)$ is replaced with $e_c(p_k^i(\bar{\mathbf{u}}^i),\bar{s}_k^i(\bar{\mathbf{v}}^i))$ in the original track boundary constraints \eqref{eq:track_boundaries} in $C_\text{kin}$. When implementing the collision avoidance constraint for the approximate dynamic game $\hat{\Gamma}_\text{kin}$, the mapping $p^i_k(\bar{\mathbf{u}}^i)$ simply extracts elements corresponding to $x$ and $y$ from the state vector of vehicle $i$ at time $k$. Finally, for this scenario, given an initial condition, we warm start the solver using trajectories obtained via a PID controller which attempt to maintain the speed $v$ and lateral deviation $e_y$ as specified in the initial condition.

\subsubsection{Scenario 2}

This scenario is identical to the previously described Scenario 1 in all aspects except for the game dynamics and warm start strategy, where instead, the vehicle is modeled using the dynamic bicycle model \cite{kong2015kinematic} and the tire forces are described via Pacejka tire models \cite{pacejka1992magic} with parameters that correspond to a 1/10th scale RC car. $\Gamma_\text{dyn}$ and $\hat{\Gamma}_\text{dyn}$ are therefore constructed in a similar manner as Scenario 1. Warm starting of the solvers is done with trajectories computed from an optimal control problem which attempts to track a pre-computed race line (the red line in the left plot of Figure~\ref{fig:exemplary_gnes}) for the two vehicles but \emph{does not} consider collision avoidance.

\subsubsection{Scenario 3}

In this scenario, we use a segment of the Austin F1 track corresponding to the first hairpin turn, as shown in the right plot of Figure~\ref{fig:exemplary_gnes}, and again model the vehicles using the dynamic bicycle model. Model parameters are chosen which match those of a full scale race car. The agent objectives in this scenario remain identical to Scenario 1 and the constraints are also the same, with the exception of the the track boundary constraints which, due to the variable track width, are now written as:
\begin{align*}
    -w_-(s_k^i(\mathbf{u}^i)) \leq e_{y,k}^i(\mathbf{u}^i) \leq w_+(s_k^i(\mathbf{u}^i)),
\end{align*}
where $w_-, w_+: [0,L] \mapsto \mathbb{R}_{++}$ are functions which return the width of the track in the directions normal to the parametric path at the given arclength $s$. We also use the same warm start strategy as in Scenario 2 with the race line shown in the right plot of Figure~\ref{fig:exemplary_gnes}.

\subsection{An Ablation Study on the Merit Function and Non-Monotone Line Search Strategy}

We first demonstrate the value of our proposed merit function and non-monotone line search strategy, by performing an ablation study on these components for the dynamic game described in Scenario 1 with a horizon of $N=25$. We compare convergence results from the full DG-SQP approach presented in Algorithm~\ref{alg:dynamic_game_sqp} with two ablative cases. In the first, we replace the merit function \eqref{eq:merit_function} with the sum of all of the agents' objectives: 
\begin{align}
    \phi(\mathbf{u}, s; \mu) = \sum_{i = 1}^{M}J^i(\mathbf{u}^1, \dots, \mathbf{u}^M) + \mu \|C(\mathbf{u})-s\|_1, \label{eq:ablation_merit}
\end{align}
which is a standard choice in nonlinear optimization approaches \cite{wright1999numerical}. In the second, we do not use the non-monotone line search strategy and instead employ a standard backtracking line search which enforces the merit decrease condition \eqref{eq:merit_decrease_condition} at every iteration. The results are shown in Table~\ref{tab:ablation_results} where GNE are solved for using DG-SQP and the two ablative cases. This is done over 100 randomly sampled joint initial conditions on the L-shaped track in Figure~\ref{fig:exemplary_gnes} where the vehicles start within 1.2 car lengths of each other and are traveling at velocities that exhibit at most a 25\% difference. We count the number of trials where the solvers reached convergence with thresholds $\rho_1,\rho_2,\rho_3 = 10^{-4}$. From the results, it can be clearly seen that certainly both the proposed merit function and the non-monotone line search strategy have a positive impact on the numerical robustness of the solver. In particular, we observe a 5.3\% and 62.8\% reduction in solver success rate when the two components are ablated. Comparing between the two components, it appears that the non-monotone line search has a larger impact on solver performance. We believe that this is due to the globalizing effect that the non-monotone scheme has on the solution process, which was shown in \cite{dirkse1995path} for the PATH solver. While we do not prove the same property for the DG-SQP solver, it is likely that the non-monotone scheme can have a similar effect.

\begin{table}[t]
\centering
\caption{Ablation study convergence results}
\begin{tabular}{c|c|c|c} 
 & no ablation & ablation merit function & ablation line search \\
\midrule[1pt]
Conv. & 94 & 89 & 35 \\
\end{tabular}
\label{tab:ablation_results}
\end{table}


\subsection{A Sensitivity Study on the Decaying Regularization Scheme}

\begin{figure}[t] 
    \centering
    \includegraphics[width=0.8\columnwidth]{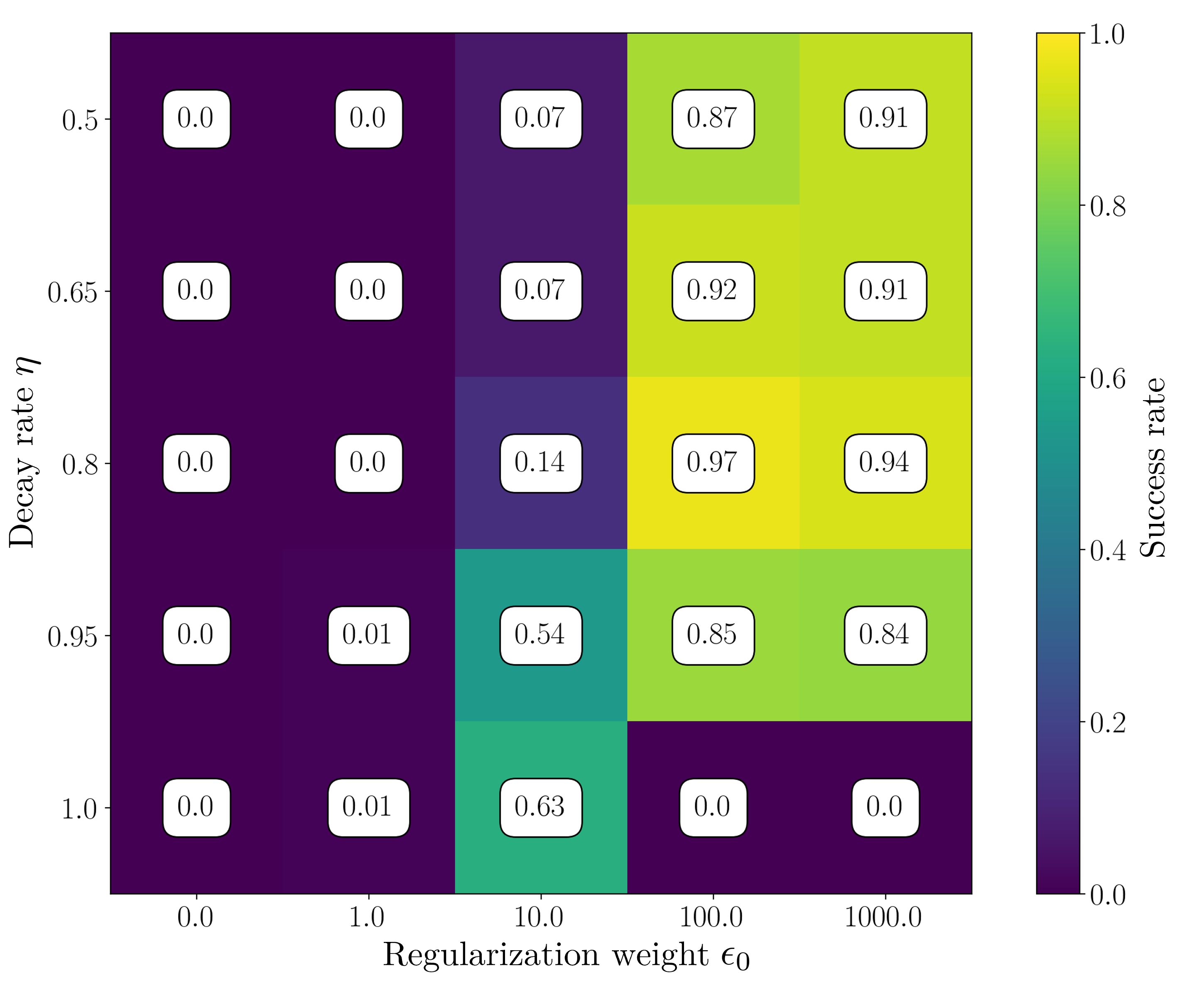}
    \vspace{-0.3cm}
    \caption{DG-SQP success rates for $\hat{\Gamma}_\text{kin}$ over different regularization weight ($x$-axis) and decay rate ($y$-axis) settings.}
    \label{fig:regularization_success}
\end{figure}

Next, we investigate the sensitivity of our approach to the regularization values and decay rates, by performing a sensitivity study over a grid of these values for the same two-player car racing dynamic game as the ablation study. In particular, we again randomly sample 100 joint initial conditions and count the number of trials where the solver reached convergence for values of $\epsilon_0 \in \{0, 0.1, 1, 10, 1000\}$ and $\eta \in \{0.5, 0.65, 0.8, 0.95, 1\}$. Note that $\epsilon_0 = 0$ corresponds to no regularization (in which case the value of the decay rate becomes irrelevant), and $\eta = 1$ corresponds to no regularization decay, which corresponds to standard regularization schemes with a constant regularizer. The results of this study are shown in Figure~\ref{fig:regularization_success}, where we may observe 1) the importance of regularization, as none of the trials reached convergence when no regularization is used (first column), and 2) the sensitivity of solver performance to the choice of a constant $\epsilon_0$, where the solver was only able to achieve a nontrivial success rate for a single setting of $\epsilon_0$ (last row). Next, by looking along the rows of Figure~\ref{fig:regularization_success}, it can be seen that decaying regularization helps to alleviate the sensitivity of solver convergence to the choice in $\epsilon_0$ as we observe nontrivial success rates over a range of magnitudes for $\epsilon_0$. Finally, by looking along the columns of Figure~\ref{fig:regularization_success}, we see that while some settings of decay rate can certainly lead to better success rates, it appears that a fairly large range of values are sufficient to stabilize the solution process to achieve high success rates of $>$80\%. While our analysis does not provide formal performance guarantees w.r.t. selection of $\epsilon_0$ and $\eta$, our sensitivity study suggests that large initial regularization values and slow decay are crucial to achieve high success rates for our solver.

\subsection{Solver and Model Comparison for Racing Tasks} \label{sec:solver_comparison}

The PATH solver \cite{dirkse1995path} is considered to be the state-of-the-art in solving for GNE of open-loop dynamic games and has been used effectively in many multi-agent navigation tasks \cite{liu2023learning,peters2023contingency}. However, these works have primarily demonstrated its performance on games with inertial kinematic models. We now compare its performance against our DG-SQP approach for the exact and approximate games $\Gamma_*$ and $\hat{\Gamma}_*$ in the head-to-head racing scenarios described above. 

\subsubsection{Scenario 1}

\begin{figure}[t] 
    \centering
    \includegraphics[width=0.99\columnwidth]{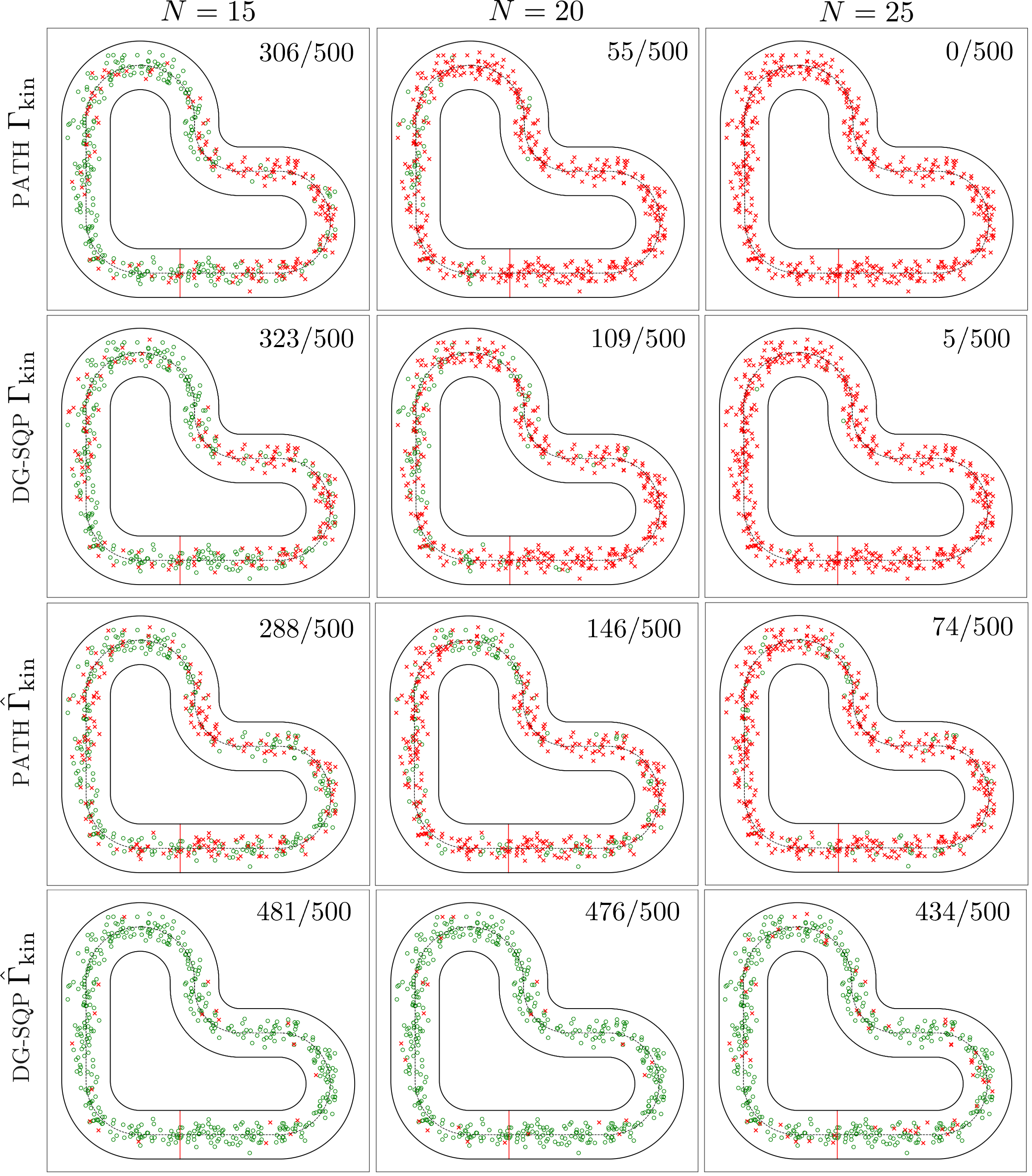}
    \caption{Results of the solver comparison study on the L-shaped track with the kinematic bicycle model from Scenario 1. Each point corresponds to the average of the sampled initial $x$-$y$ positions for the two agents. {\color{Green} $\circ$}, {\color{Red} $\times$} denote successful and failed trials respectively. The number in the top right of each plot shows the number of successful GNE solves out of the 500 sampled inital conditions.}
    \label{fig:compare_l_track_kinematic}
\end{figure}

For this scenario, we randomly sample 500 collision-free initial conditions where the agents start within 1.2 car lengths of each other and are traveling at velocities that exhibit at most a 25\% difference. For each initial condition, we solve for GNE of $\Gamma_{\text{kin}}$ and $\hat{\Gamma}_{\text{kin}}$ with horizon lengths of $N \in \{15, 20, 25\}$ using both the PATH and DG-SQP solvers. Convergence thresholds are again set to $\rho_1,\rho_2,\rho_3 = 10^{-4}$. The results are shown in Figure~\ref{fig:compare_l_track_kinematic} where we first observe that for the exact dynamic game $\Gamma_{\text{kin}}$, our DG-SQP solver shows better, performance when compared against the PATH solver, achieving higher success rates for all three horizon lengths. However, it is clear that both solvers struggle with finding a solution for longer horizon dynamic games. We note from the first column of Figure~\ref{fig:compare_l_track_kinematic}, that for the exact dynamic game $\Gamma_{\text{kin}}$, the majority of the failure cases for the PATH and DG-SQP solvers occur at the entrance of the chicane. Due to the many tight turns, this section of the track is where the numerical shortcomings of the Frenet-frame dynamics, as described in Section~\ref{sec:approximation}, can be especially apparent. Next, we examine the performance of the solvers on the approximate dynamic game $\hat{\Gamma}_{\text{kin}}$. Here, we observe significant improvements in success rate for the DG-SQP solver over all horizon lengths, with a success rate of over 85\% in the case of $N=25$. This is an $\sim$87x improvement when compared to the success rate of solving the exact dynamic game $\Gamma_{\text{kin}}$ with the same DG-SQP solver. 
When solving $\hat{\Gamma}_{\text{kin}}$ with the PATH solver, we see improvements over its performance with the exact dynamic game for the longer horizons of $N=20$ and $N=25$. However, we again see that DG-SQP achieves superior performance, with a $\sim$6x improvement compared to the PATH solver in the case of $N=25$.

\subsubsection{Scenario 2}

\begin{figure}[t] 
    \centering
    \includegraphics[width=0.99\columnwidth]{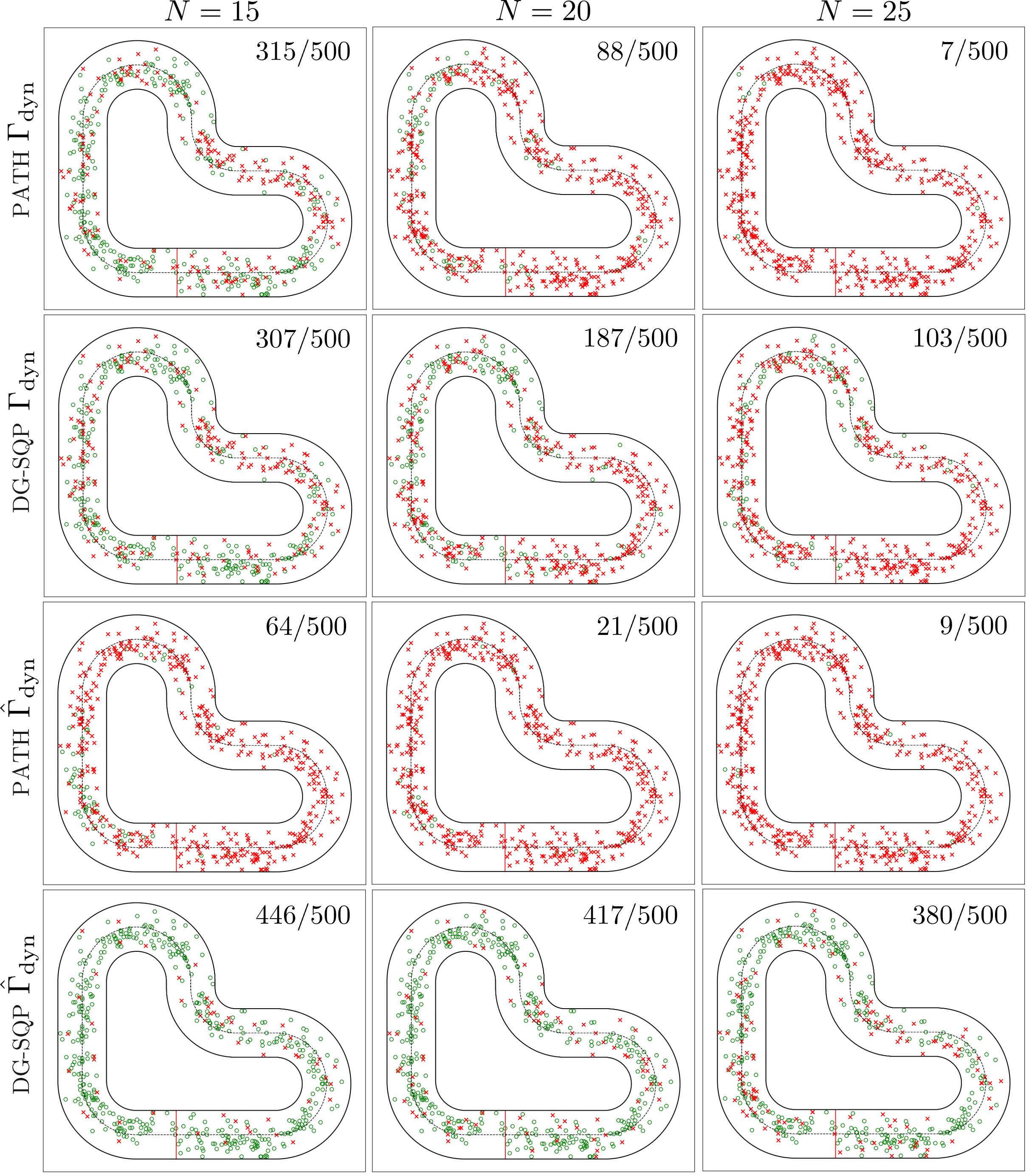}
    \caption{Results of the solver comparison study on the L-shaped track with the dynamic bicycle model from Scenario 2.}
    \label{fig:compare_l_track_dynamic}
\end{figure}

\begin{figure}[t] 
    \centering
    \includegraphics[width=0.9\columnwidth]{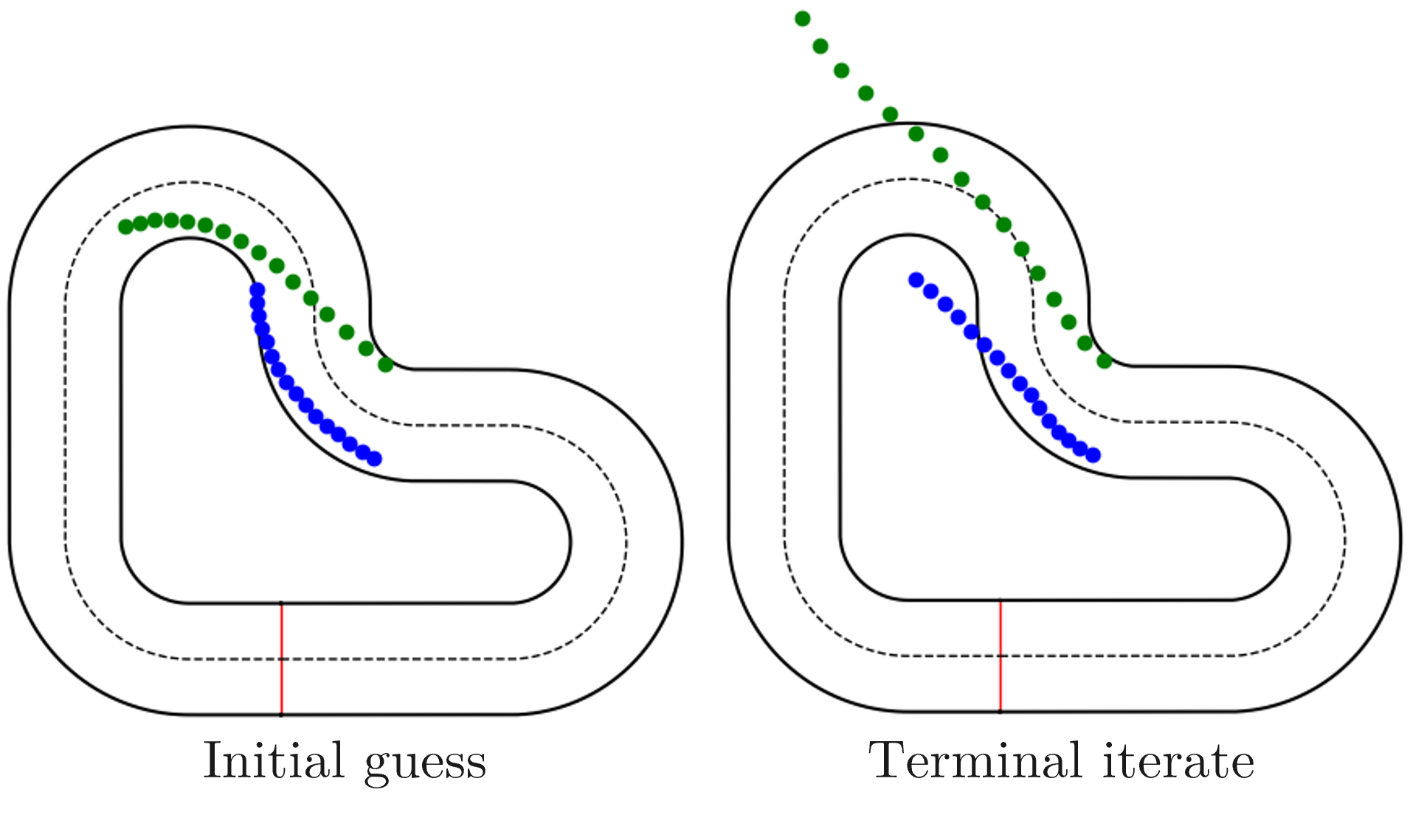}
    \caption{An example failure case by the PATH solver for the approximate game $\hat{\Gamma}_{\text{dyn}}$ from Scenario 2. The left plot shows the initial guess used to warm start the solver. The right plot shows the iterate returned by the PATH solver.}
    \label{fig:path_failure_cases}
\end{figure}


For this scenario, we randomly sample 500 collision-free initial conditions about a pre-computed race line (the red line in the left plot of Figure~\ref{fig:exemplary_gnes}) where the agents start within 1.2 car lengths of each other and at velocities which are within $\pm0.75$ m/s of the race line velocity at the sampled $s_0$. We compare the success rates of solving for GNE from these initial conditions for the exact and approximate dynamic games $\Gamma_{\text{dyn}}$ and $\hat{\Gamma}_{\text{dyn}}$. This is done again for the set of horizon lengths of $N \in \{15, 20, 25\}$. The results are summarized in Figure~\ref{fig:compare_l_track_dynamic}, where we observe similar trends to the results from Scenario 1 in the case of the exact dynamic game $\Gamma_{\text{dyn}}$, which shows that our DG-SQP solver achieves similar or superior performance to PATH for all three values of horizon lengths. Turning our attention now to the approximate dynamic game, we again see that our DG-SQP solver achieves significantly higher success rates when solving $\hat{\Gamma}_\text{dyn}$, with an over 3x improvement over the case of $\Gamma_{\text{dyn}}$ with DG-SQP and $N=25$. Interestingly, we observe a rather severe decrease in performance by the PATH solver when the dynamic bicycle model is used in the formulation of the approximate game. An example failure case is shown in Figure~\ref{fig:path_failure_cases}, where despite warm starting the solver with a feasible initial guess, the terminal iterate is infeasible. In contrast, a combination of the regularization scheme and the explicit enforcement of (linearized) constraints when computing the primal and dual step leads to  the success of the DG-SQP solver. This encourages small steps especially at the beginning of the solution process and helps to maintain the feasibility of the iterates. On the other hand, we note that the PATH solver does not have any explicit regularization scheme or constraint enforcement when computing the primal and dual step at each iterate, and instead only relies on a non-monotone line search scheme for dampening the iterates.


\subsubsection{Scenario 3}

\begin{figure}[t] 
    \centering
    \includegraphics[width=0.99\columnwidth]{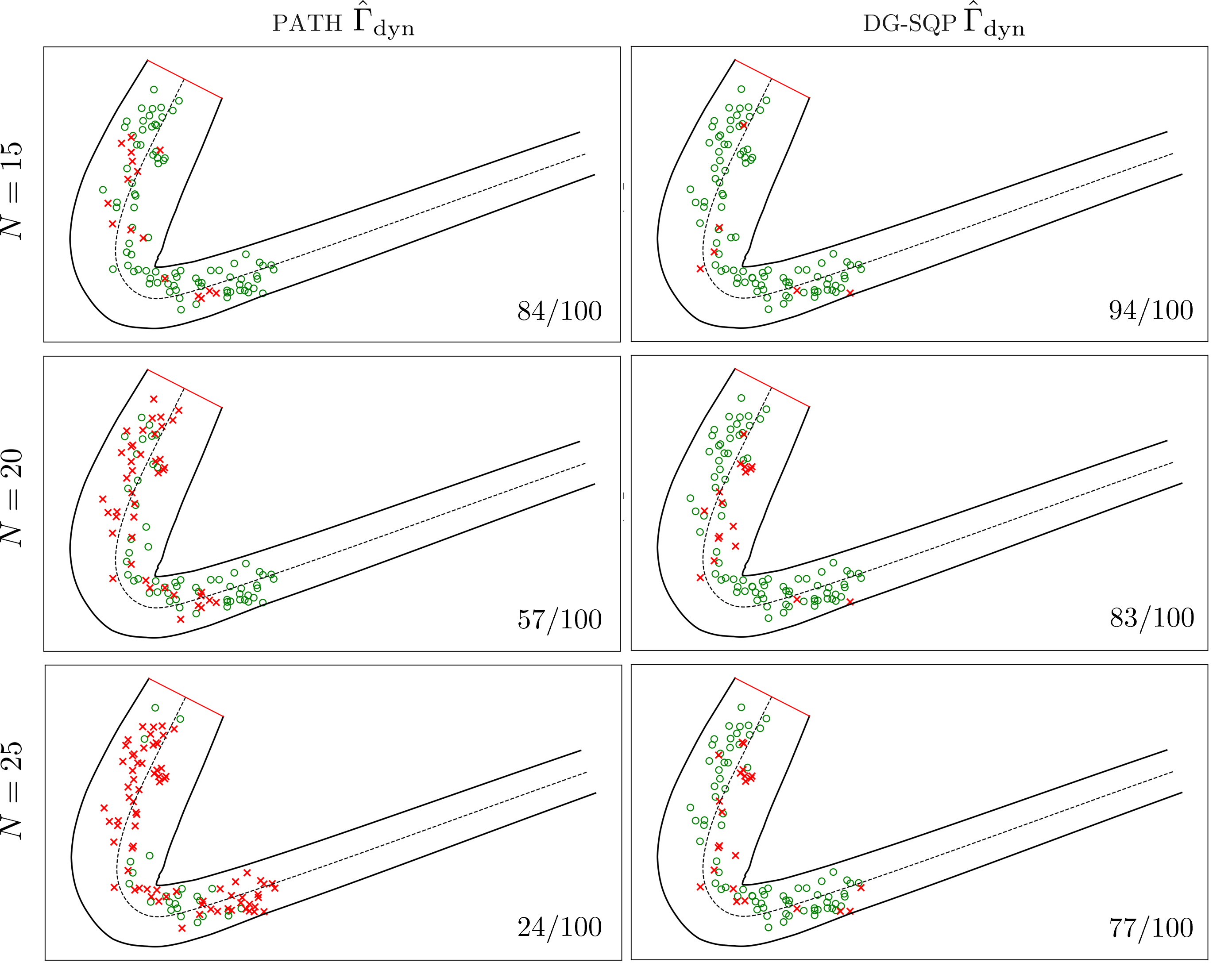}
    \caption{Results of the solver comparison study on the Austin F1 track from Scenario 3. The number in the bottom right of each plot shows the number of successful GNE solves out of the 100 sampled inital conditions.}
    \label{fig:compare_f1}
\end{figure}


Using a similar procedure to Scenario 2, we randomly sample 100 joint initial conditions about a precomputed raceline (the red line in the right plot of Figure~\ref{fig:exemplary_gnes}) where the agents start within 3 car lengths of each other and at velocities which are within $\pm7.5$ m/s of the race line velocity at the sampled $s_0$. The success rates of solving for GNE of the approximate dynamic game $\hat{\Gamma}_{\text{dyn}}$ from these initial conditions are shown in Figure~\ref{fig:compare_f1}, where we observe that, in the case of $N=25$, our DG-SQP solver is able to achieve a success rate of 77\% of the sampled initial conditions and out-performs the PATH solver by over 3x.


\subsection{Comparing GNE From the Exact and Approximate Dynamic Games}

\begin{figure}[t] 
    \centering
    \includegraphics[width=0.99\columnwidth]{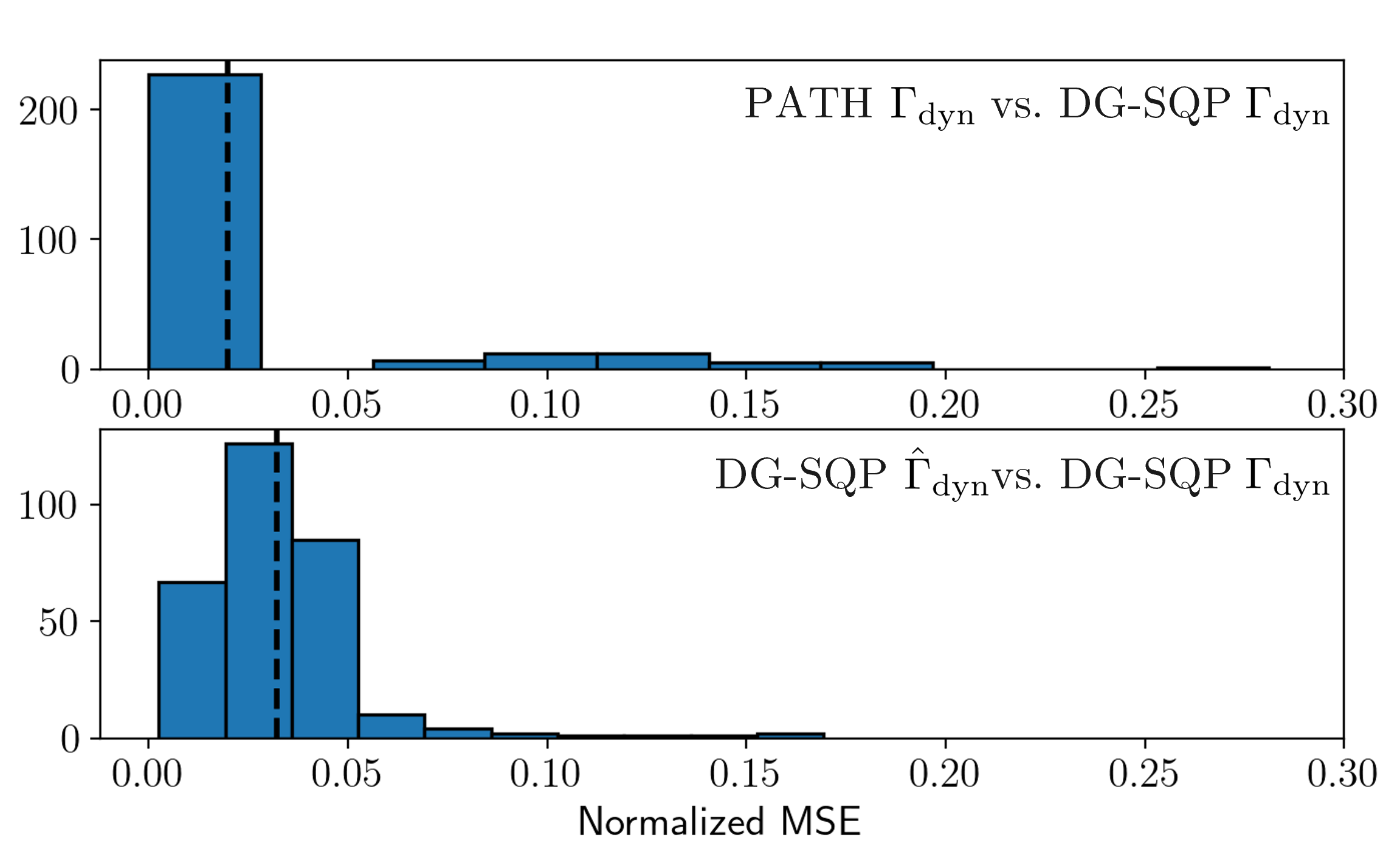}
    \caption{Distribution of the normalized mean squared error (MSE) between the GNEs from the three approaches. The dashed black line corresponds to the mean.}
    \label{fig:gne_mse_hist}
\end{figure}

We have just shown using three scenarios that significantly higher success rates can be achieved when solving for GNE of the approximate dynamic games $\hat{\Gamma}_*$ with the DG-SQP solver. However, it should be clear that as our approximation scheme modifies the components of the exact dynamic game $\Gamma_*$, we are therefore solving for the GNE of a related, but nevertheless different dynamic game. The question that naturally arises is then how the GNEs from the approximate dynamic game $\hat{\Gamma}_*$ compare with those from the exact one $\Gamma_*$. In order to answer this question, we now examine the solutions from Scenario 2 with $N=15$. Specifically, compute the pairwise mean squared error (MSE) between the GNE from two comparison cases. The first compares the GNE of  $\Gamma_\text{dyn}$ for the samples where both PATH and DG-SQP were successful. The normalized MSE for this case is defined as follows: 
\begin{align*}
\frac{1}{N}\sum_{k=0}^{N}\sum_{i=1}^{M}\|u_{DGSQP,k}^{i,\star} - u_{PATH,k}^{i,\star}\|_{D^{-1}}^2.
\end{align*}
The second case compares the GNE of $\Gamma_\text{dyn}$ and $\hat{\Gamma}_\text{dyn}$ from DG-SQP, where we only consider samples where DG-SQP was able to successfully solve for GNE of both the exact and approximate dynamic games. The normalized MSE for this case is defined as follows:
\begin{align*}
\frac{1}{N}\sum_{k=0}^{N}\sum_{i=1}^{M}\|\hat{u}_{DGSQP,k}^{i,\star}[1:2] - u_{DGSQP,k}^{i,\star}\|_{D^{-1}}^2,
\end{align*}
where $D = \text{diag}(u_u)$ normalizes the errors using the input upper bound $u_u$ from \eqref{eq:input_limits} and $\hat{u}_{DGSQP,k}^{i,\star}[1:2]$ denotes the first two elements of the GNE solution at time step $k$ (recall that $\hat{u}$ is the augmented input vector which includes the approximate arc speed $\bar{v}$). The distribution of these errors are shown in Figure~\ref{fig:gne_mse_hist}, where we may first observe from the top histogram that, when solving the exact Frenet-frame dynamic game, the DG-SQP and PATH solvers, for the most part, produce the same solutions with a minimum, median, and mean MSE of $1.57\times 10^{-6}$, $1.01\times 10^{-5}$, and $2.00\times 10^{-2}$  respectively. This serves as a baseline to illustrate the correctness of our approach. Looking at the bottom histogram, we see that although there are certainly outliers, the majority of the GNE from the approximate dynamic game are close to those from the exact dynamic game when the same solver is used. In this comparison case, we observe a minimum, median, and mean MSE of $2.54\times 10^{-3}$, $3.05\times 10^{-2}$, and $3.22\times 10^{-2}$  respectively.

We now turn to specific examples from the histograms to examine the qualitative differences in the GNE from the two comparison cases in Figures~\ref{fig:path_dgsqp_compare} and \ref{fig:dgsqpapp_dgsqp_compare}. In particular, we visualize the three GNE pairs with the smallest MSE (``Best" column), largest MSE (``Worst" column), and the three pairs with a MSE that is closest to the mean value (``Average" column). 

\begin{figure}[t] 
    \centering
    \includegraphics[width=0.99\columnwidth]{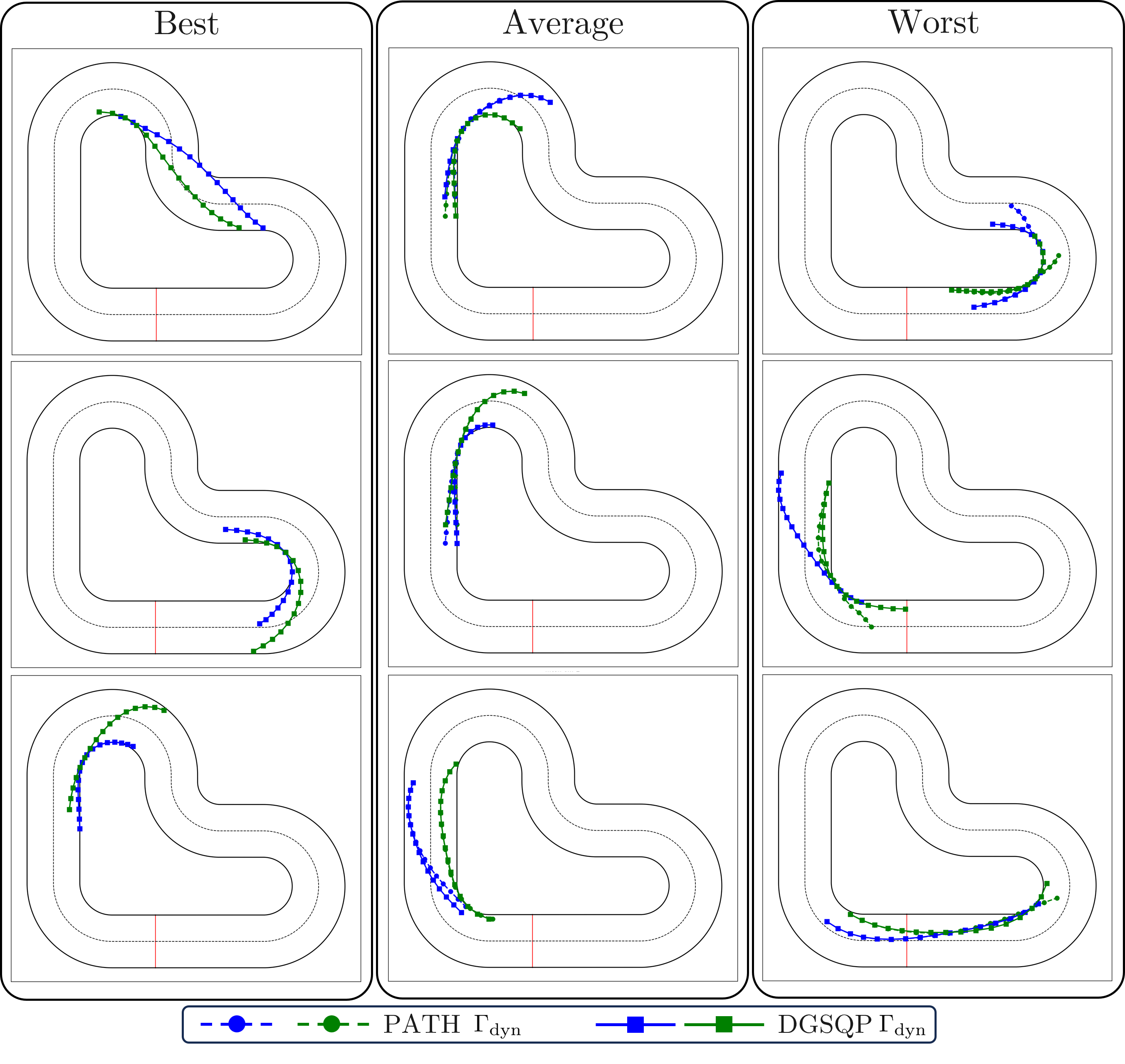}
    \caption{Comparison of GNEs of the exact dynamic game from the DG-SQP and PATH solvers. The blue and green traces represent the GNE trajectories of agents 1 and 2 respectively. The three plots in each column correspond to the top three cases in that category.}
    \label{fig:path_dgsqp_compare}
\end{figure}

In Figure~\ref{fig:path_dgsqp_compare}, we first compare the position traces arising from the GNE of the exact dynamic game $\Gamma_\text{dyn}$ from the DG-SQP and PATH solvers. In the ``Best" column, we  observe that the GNE from the two solvers are are identical. We note that in this comparison case, these examples are indicative of the majority of successful GNEs as evidenced by the minimum and median MSEs. In the ``Average" and ``Worst" columns, we observe that some differences arise especially towards the end of the game horizon. These instances correspond to the solvers finding different local GNE. However, we note that in all of the ``Average" and ``Worst" examples, the outcome, i.e. the relative ordering of the two vehicles at the end of the game horizon, is the same.


\begin{figure}[t] 
    \centering
    \includegraphics[width=0.99\columnwidth]{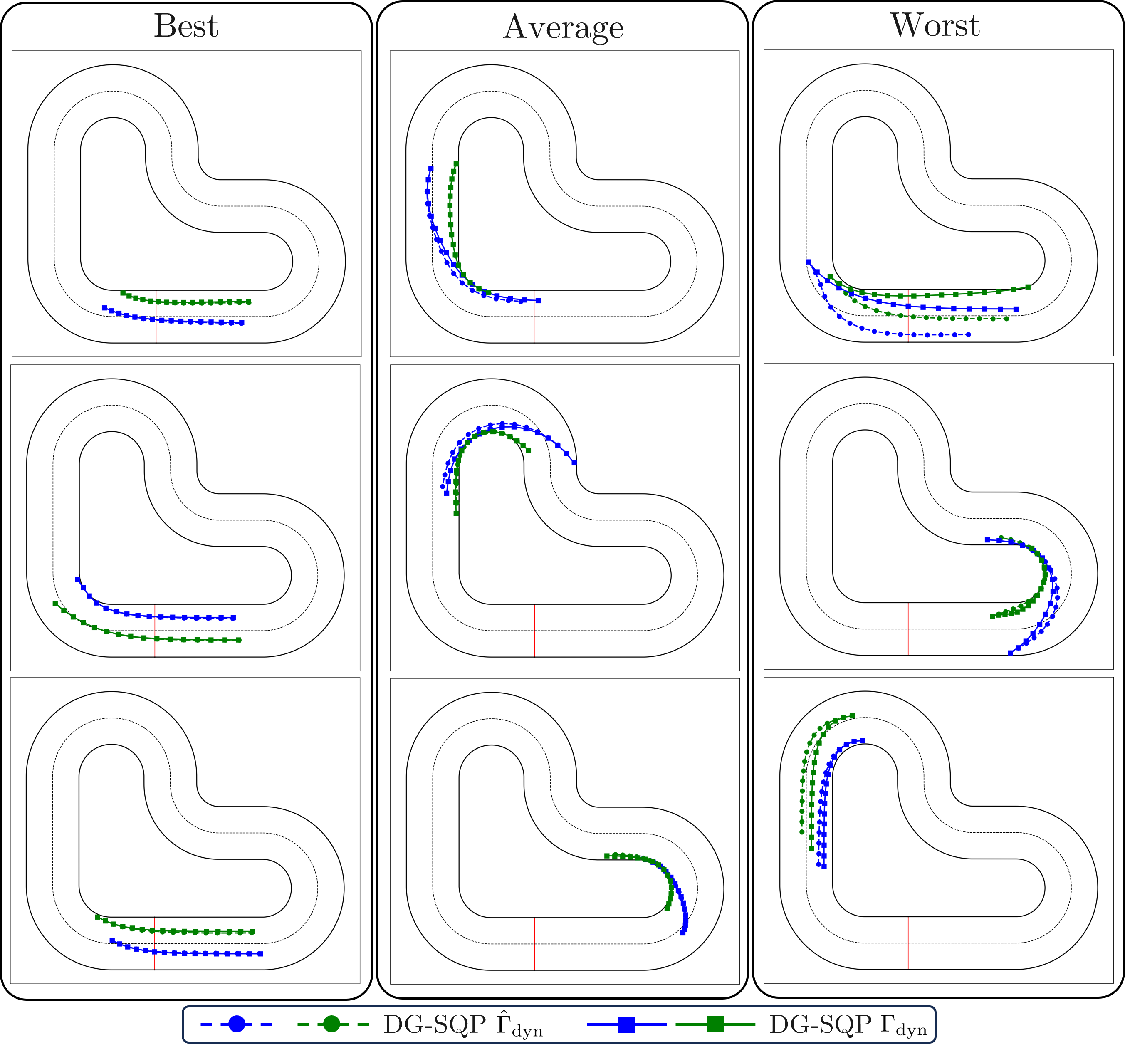}
    \caption{Comparison of GNEs of the approximate and exact dynamic games from the DG-SQP solver.}
    \label{fig:dgsqpapp_dgsqp_compare}
\end{figure}

In Figure \ref{fig:dgsqpapp_dgsqp_compare}, we compare GNE obtained with the DG-SQP solver for the exact and approximate dynamic games. We first observe that in the ``Best" column, the GNE of the approximate and exact dynamic games are identical. The fact that this occurs on a straight section of the track is unsurprising as this is where $\bar{s}$ can most accurately approximate the true path progress $s$ through the lag error in \eqref{eq:lag_error}. To see this, simply evaluate \eqref{eq:lag_error} with $\Phi = 0$, which leads to $e_l(p, \bar{s}) = \tau_x(\bar{s}) - x$. From the ``Average" column, we see that differences between the GNE are minor and that importantly, the GNE of the approximate dynamic game still capture the competitive nature of the racing scenario. Furthermore, we note that as in the previous comparison, the outcome remains the same despite small perturbations to the GNE trajectories between the exact and approximate dynamic games. This holds true even for the examples in the ``Worst" column despite larger discrepancies between the GNE of the exact and approximate dynamic games.

\section{Limitations and Future Work} \label{sec:limitations}

In this work, we have measured the performance of our DG-SQP solver against the state-of-the-art PATH solver using the metric of success rate and showed significant improvements in racing scenarios especially when approximate dynamics are used. Another metric which we have not discussed is that of solution time. This is especially important if we would like to use our DG-SQP solver in a real-time MPGP manner like in \cite{liu2023learning} and \cite{peters2023contingency}. However, we note that main shortcoming of our approach at this time is computational efficiency, with solutions taking up to two minutes for some of the successful trials with $N=25$ in Scenario 3. This is due to two reasons. The first is that our solver requires the solution of a sequence of constrained quadratic programs, which can be computationally demanding especially when compared to the PATH solver, which at its core, solves a sequence of linear systems. Secondly, our solver is implemented in Python, which is again significantly slower than the C++ implementation of the PATH solver. Though we can certainly implement our solver in a different language, we anticipate that it would still be difficult to achieve real-time solutions for the long-horizon racing problems presented here due to the solution of QPs. As such, an important direction of future work is to improve the computational efficiency of our DG-SQP solver, which would allow us to leverage its solutions for real-time racing experiments. One possible approach could be to improve the quality of the initial guess through supervised learning of GNE as in \cite{wang2020multi} and \cite{peters2022learning}. By warm starting the solver with the a neural network, which predicts GNE for a given joint initial condition, this would potentially reduce the number of solver iterations required to reach convergence. An alternative approach would be to leverage the numerical robustness of our DG-SQP solver to generate an extensive dataset of GNE from various initial conditions and track environments. This data could then be used to learn a game-theoretic value function which can be integrated in an optimal control scheme to aid in long-term strategic planning without the need for online solutions from our solver.

\section{Conclusions} \label{sec:conclusions}

In this work, we have presented DG-SQP, an SQP approach to the solution of generalized Nash equilibria for open-loop dynamic games. We show that the method exhibits local linear convergence in the neighborhood of GNE and present several practical improvements to the vanilla SQP algorithm including a non-monotone line search strategy with a novel merit function and decaying regularization scheme. We further present an approximation scheme to Frenet-frame dynamic games which can be used to improve the performance of our solver in racing scenarios. We conducted an extensive numerical study on various head-to-head racing scenarios with both kinematic and dynamic vehicle models and different race tracks. This study showed that our DG-SQP solver out-performs the state-of-the-art PATH solver in terms of success rate when applied to interactive racing scenarios.

\bibliographystyle{IEEEtran}
\bibliography{main.bib}

\begin{IEEEbiography}[{\includegraphics[width=1in,height=1.25in,clip,keepaspectratio]{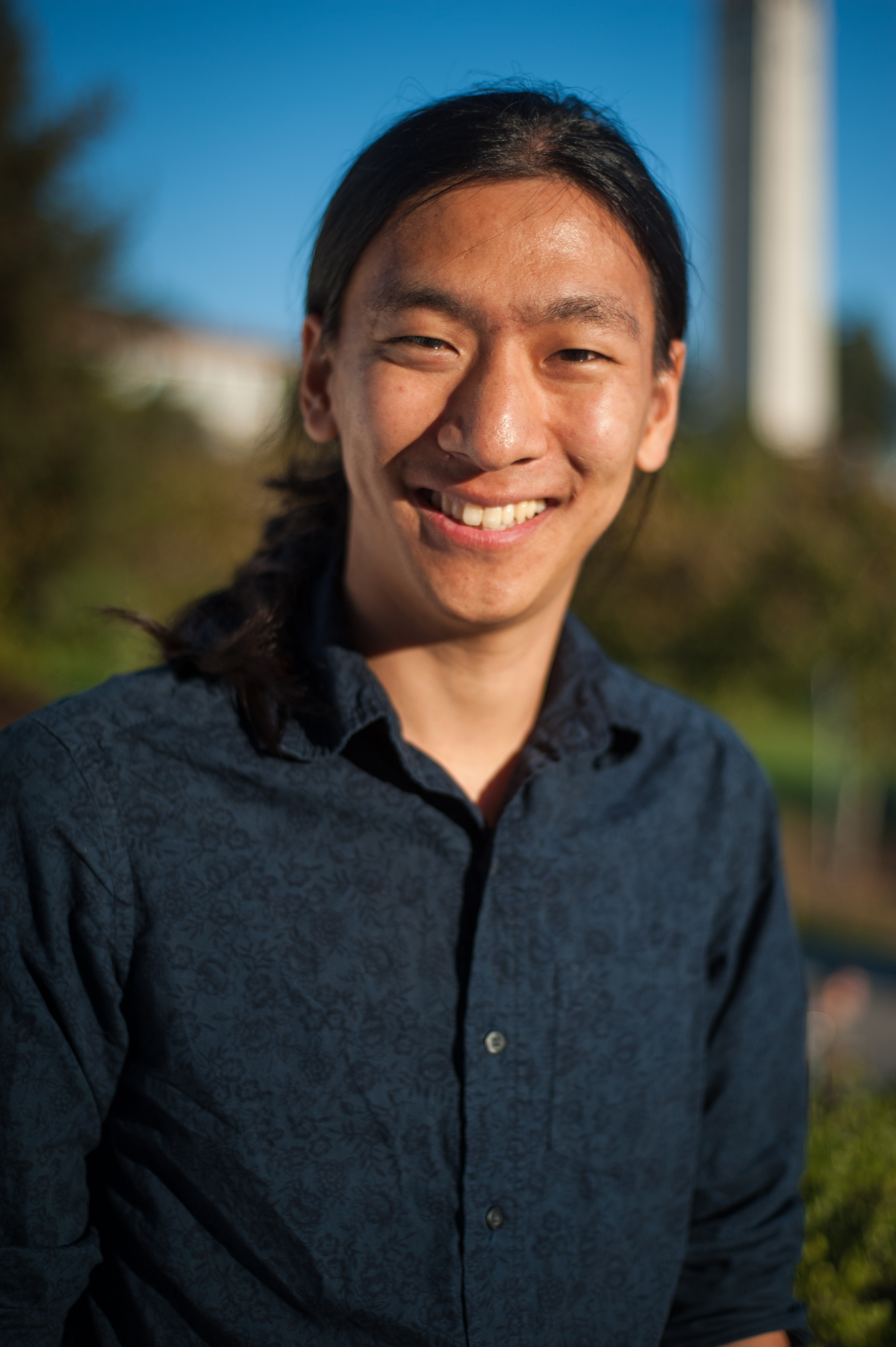}}]{Edward Zhu} received his B.S. in Mechanical Engineering from Villanova University, PA, USA, in 2015 and his M.S. in Mechanical Engineering from the University of California, Berkeley, CA, USA. He is currently a Ph.D. candidate in Mechanical Engineering at the University of California, Berkeley, CA, USA.

His research interests include leveraging model-based optimal control and game-theoretic methods for prediction and planning in multi-agent environments.
\end{IEEEbiography}

\begin{IEEEbiography}[{\includegraphics[width=1in,height=1.25in,clip,keepaspectratio]{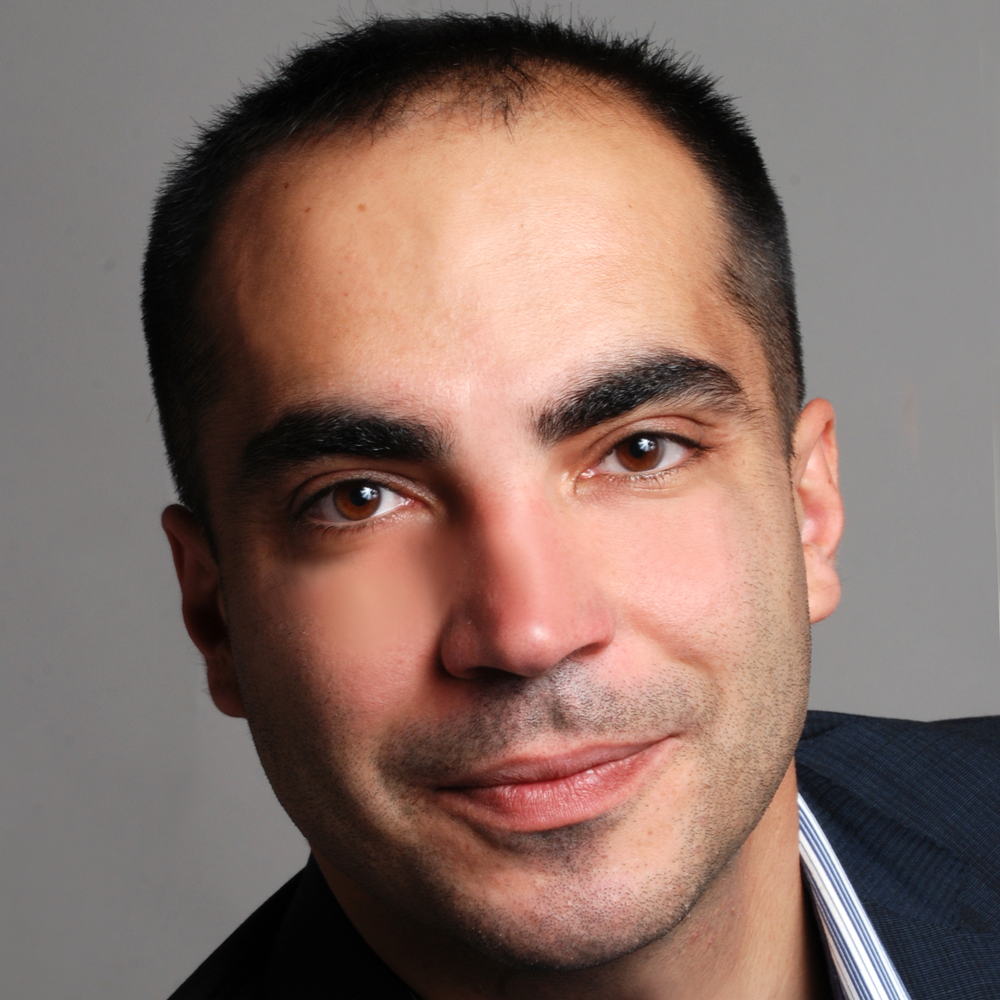}}]{Francesco Borrelli} received the Laurea degree
in computer science engineering from the University of Naples Federico II, Naples, Italy, in 1998, and the Ph.D. degree from ETH-Zurich, Zurich, Switzerland, in 2002.

He is currently an Associate Professor with the Department of Mechanical Engineering, University of California, Berkeley, CA, USA. He is the author of more than 100 publications in the field of predictive control and author of the book
Constrained Optimal Control of Linear and Hybrid Systems (Springer-Verlag). His research interests include constrained optimal control, model predictive control and its application to
advanced automotive control and energy efficient building operation.

Dr. Borrelli received the 2009 National Science Foundation CAREER
Award and the 2012 IEEE Control System Technology Award. In 2008,
he became the Chair of the IEEE Technical Committee on Automotive
Control.
\end{IEEEbiography}

\vfill

\appendix
Here, we describe the kinematic and dynamic bicycle models which are used to formulate the dynamic games in Section~\ref{sec:approximation}. 

The Frenet-frame kinematic bicycle model has the state and input vectors:
\begin{align*}
    x_{\text{kin}} = [v, s, e_y, e_\psi]^\top, \ u_{\text{kin}} = [a, \delta]^\top,
\end{align*}
and the continuous-time dynamics are written as:
\begin{align}
    \dot{v} &= a, \label{eq:s_dot} \\
    \dot{s} &= v\cos(e_\psi+\beta(\delta))/(1-e_y \kappa(s)), \nonumber \\
    \dot{e}_y &= v\sin(e_\psi+\beta(\delta)), \nonumber \\
    \dot{e}_\psi &= v\sin(\beta(\delta))/l^r - \kappa(s)v\cos(e_\psi+\beta(\delta))/(1-e_y \kappa(s)), \nonumber
\end{align}
where $\beta(\delta) = \arctan(\tan\delta\cdot l^f/(l^f+l^r))$ is the side slip angle, $l^f$, $l^r$ are the distance from the center of mass to the front and rear axles respectively, and $\kappa(s) = (\tau_x'(s)\tau_y''(s) - \tau_y'(s)\tau_x''(s))/(\tau_x'(s)^2+\tau_y'(s)^2)^{3/2}$ is the curvature of the path $\tau$ at a given $s$. 

The Frenet-frame dynamic bicycle model has the state and input vectors:
\begin{align*}
    x_{\text{dyn}} = [v_x, v_y, \omega, s, e_y, e_\psi]^\top, \ u_{\text{dyn}} = [a_x, \delta]^\top,
\end{align*}
and the continuous-time dynamics are written as:
\begin{align*}
    \dot{v}_x &= a_x - \frac{1}{m}F_y^f\sin\delta - c_d v_x + \omega v_y, \\
    \dot{v}_y &= \frac{1}{m}(F_y^r + F_y^f\cos\delta) - \omega v_x, \\
    \dot{\omega} &= \frac{1}{I_z}(l^f F_y^f \cos\delta - l^r F_y^r), \\
    \dot{s} &= (v_x \cos(e_\psi) - v_y \sin(e_\psi))/(1-e_y \kappa(s)), \\
    \dot{e}_y &= v_x \sin(e_\psi) + v_y \cos(e_\psi), \\
    \dot{e}_\psi &= \omega - \kappa(s)(v_x \cos(e_\psi) - v_y \sin(e_\psi))/(1-e_y \kappa(s)),
\end{align*}
where $m$ and $I_z$ are the mass and yaw moment of inertia of the vehicle and $c_d$ is an aerodynamic drag coefficient. The lateral tire forces are modeled using a simplified Pacejka tire model \cite{pacejka1992magic}:
\begin{align*}
    F_y^f &= D^f \sin (C^f + \arctan(B^f \alpha^f)) \\
    F_y^r &= D^r \sin (C^r + \arctan(B^r \alpha^r)),
\end{align*}
where $B$, $C$, and $D$ are parameters obtained experimentally and $\alpha^f$ and $\alpha^r$ are the slip angles of the front and rear tires respectively:
\begin{align*}
    \alpha^f &= -\arctan\left(\frac{\omega l^f + v_y}{v_x}\right) + \delta \\
    \alpha^r &= \arctan\left(\frac{\omega l^r - v_y}{v_x}\right).
\end{align*}

The inertial-frame kinematic bicycle model has the state and input vectors:
\begin{align}
    \bar{x}_{\text{kin}} = [v, x, y, \psi]^\top, \    \bar{u}_{\text{kin}} = [a, \delta]^\top,
\end{align}
and the continuous-time dynamics are written as:
\begin{align*}
    \dot{v} &= a, \\
    \dot{x} &= v\cos(\beta(\delta) + \psi), \\
    \dot{y} &= v\sin(\beta(\delta) + \psi), \\
    \dot{\psi} &= v\sin(\beta(\delta))/l^r, 
\end{align*}
    
The inertial-frame dynamic bicycle model has the state and input vectors:
\begin{align*}
    \bar{x}_{\text{dyn}} = [v_x, v_y, \omega, x, y, \psi]^\top, \    \bar{u}_{\text{dyn}} = [a_x, \delta]^\top,
\end{align*}
and the continuous-time dynamics are written as:
\begin{align*}
    \dot{v}_x &= a_x - \frac{1}{m}F_y^f\sin\delta - c_d v_x + \omega v_y, \\
    \dot{v}_y &= \frac{1}{m}(F_y^r + F_y^f\cos\delta) - \omega v_x, \\
    \dot{\omega} &= \frac{1}{I_z}(l^f F_y^f \cos\delta - l^r F_y^r), \\
    \dot{x} &= v_x\cos\psi - v_y\sin\psi, \\
    \dot{y} &= v_x\sin\psi + v_y\cos\psi, \\
    \dot{\psi} &= \omega.
\end{align*}
We obtain the discrete-time Frenet-frame dynamics $f_{\text{kin}}$, $f_{\text{dyn}}$, and inertial-frame dynamics $\bar{f}_{\text{kin}}$, $\bar{f}_{\text{dyn}}$ via 4-th order Runge-Kutta discretization with a time step of $\Delta t = 0.1s$.



\end{document}